\documentclass[conference,compsoc]{IEEEtran}
\IEEEoverridecommandlockouts
\ifCLASSOPTIONcompsoc
  \usepackage[nocompress]{cite}
\else
  \usepackage{cite}
\fi
\usepackage{amsthm}
\usepackage{amsmath}
\usepackage{amsfonts}
\usepackage{xspace}
\usepackage{graphicx}
\usepackage[most]{tcolorbox}
\usepackage{enumitem}
\usepackage{multirow}
\usepackage{url}
\usepackage{subfig}
\usepackage{booktabs}
\usepackage{adjustbox}

\newcommand{\mybpara}[1]{\vspace{3pt}\noindent{\textbf{#1}}\xspace}
\newcommand{\ring}{\ensuremath{\textsc{Ring}}\xspace}

\newtheorem{theorem}{Theorem}
\newtheorem{definition}{Definition}

\usepackage{lettrine}

\usepackage{xcolor}
\usepackage[table]{xcolor}

\ifCLASSINFOpdf
\else
\fi
\renewcommand{\footnoterule}{%
  \kern -3pt
  \hrule width 0.4\columnwidth height 0.4pt
  \kern 2.6pt
}

\newcommand{\blfootnote}[1]{%
  \begingroup
  \renewcommand{\thefootnote}{}%
  \footnote{#1}%
  \addtocounter{footnote}{-1}%
  \endgroup
}

\begin{document}
%
\title{Your Privacy My Cloak: Backdoor Attacks \\on Differentially Private Federated Learning}

\author{\IEEEauthorblockN{Xiaolin Li}
 \IEEEauthorblockA{Purdue University\\
 West Lafayette, IN, USA\\
 li4955@purdue.edu}
 \and
 \IEEEauthorblockN{Ning Wang}
 \IEEEauthorblockA{University of South Florida\\
 Tampa, FL, USA\\
 ningw@usf.edu}
 \and
 \IEEEauthorblockN{Ninghui Li}
 \IEEEauthorblockA{Purdue University\\
 West Lafayette, IN, USA\\
 ninghui@purdue.edu}
\and
\IEEEauthorblockN{Wenhai Sun}
\IEEEauthorblockA{Purdue University\\
 West Lafayette, IN, USA\\
 whsun@purdue.edu}}

\maketitle

\begin{abstract}
Prior research suggests that differential privacy (DP) can enhance the robustness of federated learning (FL) against backdoor attacks. In this paper, we challenge this assumption. Through an empirical analysis of two baseline attack strategies, we uncover a fundamental tension in DP-FL: while bypassing DP allows state-of-the-art defenses to detect and filter malicious updates, complying with DP inadvertently masks their distinguishing statistical characteristics. Consequently, existing defenses become ineffective as DP reduces the raw backdoor signal. Building on this masking effect, we propose \ring, a novel attack that explicitly exploits DP to conceal malicious contributions while maximizing attack impact. By collaboratively crafting adversarial perturbations, compromised clients reconstruct a strong backdoor signal during aggregation without triggering anomaly detection. \ring operates as a perturbation layer that is agnostic to the underlying backdoor technique, making it broadly applicable and composable with existing attacks -- a property that significantly amplifies the threat it poses to DP-FL. Extensive evaluations across four image and text datasets under non-iid distributions show that \ring achieves an average attack success rate of $90.3\%$ against six state-of-the-art defenses under a moderate privacy budget, an improvement of up to $38.4\%$ over baseline strategies. Finally, we evaluate potential countermeasures and find that mitigating this threat incurs significant utility trade-offs, exposing a fundamental security gap in the deployment of differentially private FL.
\end{abstract}


%
\IEEEpeerreviewmaketitle

\section{Introduction}
Federated learning (FL)~\cite{mcmahan2017communication} is a distributed machine learning paradigm that enables collaborative model training while preserving data locality, with wide applications \cite{hard2018federated,yang2018applied,ramaswamy2019federated,brisimi2018federated,ogier2022flamby}. Yet FL remains vulnerable to backdoor attacks~\cite{wang2020attack,bagdasaryan2020backdoor}, in which compromised clients inject poisoned updates to cause targeted misclassifications at inference time. Applying differential privacy (DP)~\cite{dwork2014algorithmic} to FL has been widely advocated as a remedy: prior work demonstrates that DP noise provides both empirical and certified robustness against such threats~\cite{Naseri2020LocalAC,sun2019can,xie2023unraveling}.
\blfootnote{This paper was accepted by IEEE S\&P 2027.}

\begin{figure}[t]
\centering
\scalebox{1}[0.8]{
\includegraphics[scale=0.45]{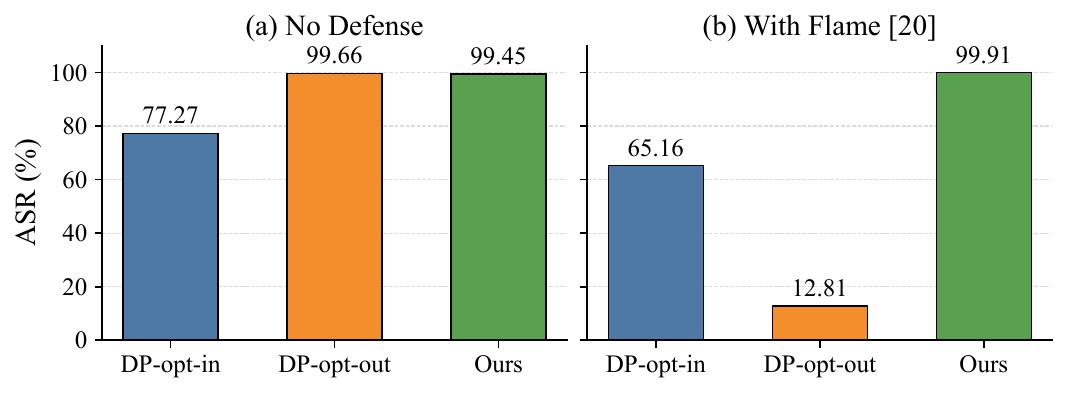}
}
\caption{ASR under sample-level DP-SGD ($\epsilon=1$) and prob non-iid setting on MNIST.}
 \label{fig_RING_asr_flame}
 \vspace{-10pt}
\end{figure}

In this paper, we challenge this assumption. We show that, through a carefully designed perturbation function, backdoor attacks can be \emph{strengthened} rather than suppressed under differentially private FL. Specifically, our attack recovers the attack success rate (ASR) comparable to an undefended setting, while simultaneously evading state-of-the-art in-training defenses. This finding connects to a broader body of work showing that DP mechanisms can be exploited to inadvertently benefit adversaries~\cite{li2023fine,cao2021data,li2024robustness,Li2025MDPA,giraldo2020adversarial}. Our results are particularly alarming given the growing deployment of DP in privacy-preserving AI systems. We demonstrate that attackers can target more advanced and consequential machine learning tasks than previously shown, raising serious questions about the sufficiency of DP as a 
backdoor defense.



\mybpara{Our Attack.} To characterize the threat landscape and motivate our research, we empirically study two intuitive baseline attack strategies, i.e., the \emph{DP-opt-in attack} and the \emph{DP-opt-out attack}, against differentially private FL protected by six state-of-the-art backdoor defenses \cite{rieger2022deepsight,blanchard2017machine,nguyen2022flame,krauss2023mesas,fereidooni2023freqfed,kabir2024flshield}. Under the DP-opt-in strategy, malicious clients perturb their model updates in compliance with DP protocols (e.g., DP-SGD~\cite{abadi2016deep}). This baseline reflects the scenario where DP operates as intended, and thus represents an upper bound on the system's robustness against backdoor attacks. In contrast, a DP-opt-out adversary bypasses DP entirely, submitting unperturbed updates directly to the aggregation server, a more pragmatic strategy that trades privacy-protocol compliance for maximized attack gain.

These experiments yield two key observations. First, \textit{random DP noise can inadvertently aid the attacker in evading detection.} Even though the DP-opt-in attack significantly reduces backdoor efficacy, the injected noise simultaneously masks the statistical fingerprints of poisoned updates, enabling them to evade existing defenses. Second, \textit{ASR recovers when the attacker opts out of DP, at the cost of increased detection risk.} Without DP perturbation, the statistical divergence between malicious and benign updates becomes pronounced, making anomaly-based detection substantially more effective. Taken together, the two baseline attacks reside opposite ends of the threat spectrum -- trading off \textit{stealthiness} against \textit{attack effectiveness} -- a tension we investigate in detail in Section~\ref{sec:threat_investigation}.

In this paper, we investigate the following research question: \emph{Can an attacker simultaneously achieve high attack success rate and evade existing backdoor defenses in differentially private FL?} Equivalently, from the defender's perspective: \textit{Does combining DP with existing mitigation methods provide sufficient protection against backdoor attacks?} To address this, our attacker jointly optimizes two objectives: \textit{effectiveness}, i.e., restoring a strong backdoor signal in the global model, and \textit{stealthiness}, i.e., reducing the detectability of malicious updates under existing defenses. Specifically, malicious clients collaboratively craft adversarial noise such that the noise masks poisoned updates locally, yet cancels out during aggregation, allowing the backdoor signal to recover in the global model while individual malicious updates remain undetected. 

To realize this, we are inspired by secret sharing~\cite{delfs2002introduction}, treating the poisoned model updates as the secret to be concealed. We name our attack \textbf{\ring}, drawing an analogy to a covert spy ring that operates under cover to infiltrate and manipulate a target. \ring is designed as an adversarial perturbation layer that is decoupled from any specific backdoor attack strategy, making it broadly applicable. It can augment existing backdoor attacks to remain effective in differentially private FL settings.

We support our claims through both theoretical analysis and extensive empirical evaluation. Theoretically, we analyze how key attack parameters shape the effectiveness-stealthiness trade-off of \ring. Our experiment shows that \ring maintains high ASR under the existing defense, achieving performance comparable to the DP-opt-out strategy in an undefended setting, as shown in Figure~\ref{fig_RING_asr_flame}. Further evaluation across three benchmark image datasets and one text dataset under three non-iid data partitions finds that \ring achieves a consistently high average ASR of $90.3\%$ against six state-of-the-art defenses -- a 1.74$\times$ improvement over the baseline attacks. 

We further conduct an ablation study to isolate the impact of attacker-controlled factors (e.g., the number of malicious clients), defender-controlled factors (e.g., the DP privacy budget and clipping bound), data heterogeneity (iid vs. non-iid settings) and attack generality. Across all configurations, \ring maintains high and stable attack performance, demonstrating resilience to a wide range of practical deployment conditions.

Finally, we explore potential countermeasures against \ring. While several candidate defenses show promise, each incurs a non-negligible utility or privacy cost, and none offers a practical, effective solution. These limitations highlight the need for principled defenses that can simultaneously preserve model utility, enforce privacy guarantees, and resist adaptive adversaries of the kind introduced here.



\mybpara{Contributions.} We make the following contributions.
\begin{itemize}[leftmargin=*]
\item We present the first systematic study of backdoor threats in differentially private federated learning, revealing that DP noise affects not only attack effectiveness but also attack stealthiness against existing defenses, an observation that directly motivates our attack design. 

\item We propose \ring, the first backdoor attack that remains effective against state-of-the-art defenses in DP-protected FL. Each malicious update is crafted to statistically resemble a DP-perturbed update, reducing its detectability under existing defenses, while their aggregation reconstructs the backdoor signal to maximize attack success.



\item We evaluate \ring across four benchmark datasets spanning image and text tasks under three non-iid data distributions, assessing it against six state-of-the-art backdoor defenses that represent diverse mitigation strategies. \ring consistently evades all defenses and achieves substantially higher ASR than baseline attacks.


\item We systematically evaluate candidate defenses against \ring and find that, while some show partial effectiveness, all incur significant utility or privacy costs and fail to generalize across settings. These results expose fundamental gaps in current defenses and raise important questions about the security of DP-protected AI/ML systems.

\end{itemize}

\section{Background and Related Works}
\subsection{Differential Privacy}
\vspace{-5pt}
DP provides rigorous statistical guarantees of information leakage from the output of a randomized algorithm.  Formally, 
\begin{definition}[$(\epsilon,\delta)$-Differential Privacy~\cite{dwork2014algorithmic}]
A randomized mechanism $M$ satisfies $(\epsilon, \delta)$-differential privacy if, for any two neighboring datasets $D_1$ and $D_2$ that differ in at most one record, and for all sets of possible outputs $A \subseteq \text{Range}(M)$, $P[M(D_1)\in A]\leq e^\epsilon P[M(D_2)\in A]+\delta \nonumber$
where $\epsilon$ and $\delta$ are privacy parameters; smaller values of both indicate a stronger privacy guarantee.
\end{definition}

In the machine learning setting, a DP training mechanism bounds how much the released model parameters can reveal about any individual training record.

\subsection{Federated Learning}\label{sec:Fed_Learn}
FL enables collaborative model training across distributed participants without centralizing raw data~\cite{mcmahan2017communication}. We adopt a standard FL setup following prior work~\cite{bagdasaryan2020backdoor,xie2019dba,fereidooni2023freqfed,kabir2024flshield}: each client holds a private local dataset over a shared feature space, and a central server coordinates training by aggregating client model updates into a global model. Raw data never leaves client devices; only model parameters are communicated to the server.

Suppose there are $n$ clients in total. At each communication round $t$, the server selects a subset $\mathcal{U}_t$ of $r \cdot n$ clients, where $r$ is the participation fraction. Each selected client $i \in \mathcal{U}_t$ initializes from the current global model $G_t$ and computes a local gradient over its private dataset $\mathcal{D}_i$ as $g_{i,t} = \nabla \mathcal{L}_{\text{task}}(G_t, \mathcal{D}_i)$. 
Each client then submits $g_{i,t}$ to the server, which aggregates the received gradients to produce the updated global model for round $t+1$:
\begin{equation}\label{eq:weighted_fedavg}
G_{t+1} = G_t - \eta \cdot \sum_{i \in \mathcal{U}_t} w_i g_{i,t}
\end{equation}
where $\eta$ is the global learning rate and $w_i$ is the aggregation weight assigned to client $i$ by the server. When a robust aggregation mechanism is employed (Section~\ref{sec:backdoor_defense}), $w_i$ is determined by that defense; otherwise, $w_i$ is set proportional to the local dataset size $|\mathcal{D}_i|$.

\subsection{Differentially Private Federated Learning}\label{sec:SL-DP}
Sharing intermediate model updates in FL can expose sensitive user information through gradient inversion and related inference attacks~\cite{zhu2019deep}. Differentially private federated learning (DP-FL)~\cite{geyer2017differentially,Naseri2020LocalAC} addresses this by incorporating formal DP guarantees into the training process. Under our assumption of a semi-honest server, which observes aggregated updates but does not manipulate the protocol, the \emph{sample-level DP} (SL-DP) framework is well-suited to our setting~\cite{gu2025dp}.

The standard mechanism for enforcing SL-DP is \emph{differentially private stochastic gradient descent} (DP-SGD)~\cite{abadi2016deep}. In each communication round $t$, client $i$ samples a mini-batch $\mathcal{B}_t$ from its local dataset $\mathcal{D}_i$ with sampling probability $\tfrac{b}{|\mathcal{D}_i|}$, where $b$ is the batch size. The per-sample gradient $g_t(x)$ for each $x \in \mathcal{B}_t$ is clipped to a fixed $\ell_2$ norm bound $C$ as $\bar{g}_t(x) = g_t(x)/\max\left(1, \tfrac{\|g_t(x)\|_2}{C}\right)$. Clipping ensures that the $\ell_2$ sensitivity of the batch gradient $\sum{x \in \mathcal{B}_t} g_t(x)$ is bounded by $C$ for any two neighboring datasets. Gaussian noise scaled to $C$ is then added to enforce $(\epsilon, \delta)$-DP, producing the perturbed gradient:
\begin{equation}\label{eq:input_backdoor}
\hat{g}_{i,t} = \frac{1}{b} \left(\sum_{x \in \mathcal{B}_t} \bar{g}_t(x) + \mathcal{N}(0, \sigma^2 I_d)\right),
\end{equation}
where $\sigma := \sigma_m \cdot C$ is the effective noise standard deviation, with $\sigma_m$ denoting the noise multiplier.

\subsection{Backdoor Attacks and Mitigations in FL}  
\subsubsection{Common Backdoor Attacks in FL}\label{sec:backdoor_background}
Backdoor attacks in FL occur when malicious clients inject a hidden task into the global model. Following prior work~\cite{wang2020attack,xie2019dba,bagdasaryan2020backdoor}, we focus on \textit{targeted model poisoning attacks}, where the adversary's goal is to cause the global model to consistently predict a predefined target label on trigger-bearing inputs~\cite{huang2011adversarial}. We assume the adversary directly controls a subset of clients and can deviate arbitrarily from the FL protocol. Concretely, each malicious client trains a backdoored gradient $g_{adv,t} = \gamma \nabla \mathcal{L}_{task}(G_t, \mathcal{D}_{\mathrm{adv},i})$, where $\mathcal{D}_{\mathrm{adv},i}$ denotes the attacker-controlled local dataset of client $i$, and $\gamma$ scales the strength of the injected backdoor signal.

Existing backdoor attacks in FL vary along two axes: how the trigger is embedded in the data, and how the malicious update is injected into the global model. On the data side, common strategies include: \textit{visible-trigger backdoor}~\cite{gu2017badnets}, which stamps a visible pattern onto inputs and flips their labels to a target class; \textit{distributed backdoor attack (DBA)}~\cite{xie2019dba}, which partitions a global trigger into local sub-patterns across clients; \textit{edge-case backdoor}~\cite{wang2020attack}, which poisons rare samples unlikely to appear in standard training; and \textit{semantic backdoor}~\cite{bagdasaryan2020backdoor}, which targets samples sharing a semantic feature (e.g., cars of a specific color) without modifying inputs directly. On the injection side, model-replacement attacks~\cite{bagdasaryan2020backdoor} amplify malicious updates to overwrite the global model, while Neurotoxin~\cite{pmlr-v162-zhang22w} concentrates the backdoor signal on parameters least frequently updated by benign clients, improving persistence across rounds.

Our approach is agnostic to the specific backdoor technique employed, as we demonstrate theoretically in Section~\ref{sec:stealthy} and empirically across multiple backdoor baselines in Section~\ref{sec:results_ring}. For our main experiments, we adopt the visible-trigger backdoor for image datasets and edge-case backdoor for text tasks unless otherwise specified.  Though these attacks are straightforward and detectable under SOTA defenses, they remain effective under the cover of \ring.

\subsubsection{Defenses}\label{sec:backdoor_defense}
Many defense schemes have been proposed to mitigate backdoor attacks during the training phase in FL. We organize them into three categories relevant to our work.

\textbf{Robust Aggregation.}
These methods identify and suppress malicious updates at the aggregation step. Multi-Krum~\cite{blanchard2017machine} (hereafter Krum) computes pairwise Euclidean distances among local updates and excludes those with the largest cumulative distances as likely malicious. MESAS~\cite{krauss2023mesas} applies six complementary statistical metrics, including cosine and Euclidean distance, to cluster updates and isolate malicious contributions. Flame~\cite{nguyen2022flame} clusters updates by pairwise cosine similarity, removes outliers, and further clips and perturbs remaining updates to limit residual adversarial influence. FreqFed~\cite{fereidooni2023freqfed} transforms model updates into the frequency domain, extracts core frequency components, and clusters them to identify poisoned updates.

\textbf{Representation-Level Defenses.}
Rather than inspecting raw parameter values, these approaches assess model behavior using auxiliary data or peer validation. DeepSight~\cite{rieger2022deepsight} computes behavioral metrics, such as normalized update energies and prediction divergences, against an auxiliary dataset to distinguish local models from the global model. FLShield~\cite{kabir2024flshield} and CrowdGuard~\cite{rieger2022crowdguard} leverage clients as validators. Each client evaluates candidate models on its local data and votes on whether a model appears benign or malicious to enable collaborative detection at the server.

\textbf{DP-enabled Robustness.}
Naseri~\textit{et al.}~\cite{Naseri2020LocalAC} empirically show that DP improves robustness against backdoor attacks while preserving acceptable utility. Sun~\textit{et al.}~\cite{sun2019can} introduce weak DP as a lightweight empirical defense that suppresses backdoor signals without significantly degrading model performance. Going beyond empirical evaluation, Xie~\textit{et al.}~\cite{xie2023unraveling} provide a theoretical analysis of DP-FL and establish certified robustness guarantees against backdoor attacks.

\textbf{Identified Gap.}
Prior work treats DP and dedicated backdoor defenses largely in isolation. This leaves two critical questions unanswered: Does DP noise inadvertently help attackers evade detection-based defenses by masking the statistical signature of poisoned updates? And conversely, can an adaptive attacker who bypasses DP still succeed in the presence of state-of-the-art mitigations? Understanding this interplay is essential for accurately characterizing the threat landscape in DP-FL, and it is precisely this gap that motivates the design of \ring.

\vspace{-5pt}
\section{Threat Investigation}\label{sec:threat_investigation}

\subsection{Threat Model}\label{sec:threat_model}
We now define the threat model that applies to both the baseline investigation in this section and the \ring attack developed in Section~\ref{sec:RING}. We consider a DP-FL setting in which benign clients exploit DP-SGD to protect their local data from a semi-honest server, which performs weighted FedAvg aggregation as described in Section~\ref{sec:Fed_Learn}. A fraction of clients may be compromised and act as adversaries.

\mybpara{Attacker's Goal.}
The attacker jointly optimizes for \textit{effectiveness} and \textit{stealthiness}. In particular, \ring targets an attack success rate comparable to that of the underlying backdoor attack in an undefended system, while simultaneously evading state-of-the-art in-training defenses deployed by the server.

\mybpara{Attacker's Capability.}
We assume the adversary controls a small fraction $\beta$ of clients. Since \ring is a coordinated attack, analogous to DBA~\cite{xie2019dba}, it requires at least two malicious clients to jointly perform the adversarial perturbation. Each compromised client has full access to its local training data and complete control over the local training process, including hyperparameters such as the learning rate and number of local epochs. The attacker may also arbitrarily modify local model weights before submission to the server. This capability assumption is standard in the backdoor FL literature~\cite{rieger2022deepsight,rieger2022crowdguard,fereidooni2023freqfed}. Importantly, the attacker has no knowledge of whether a defense is deployed or how it operates, making the attack fully black-box with respect to the server-side mitigation.

We evaluate two baseline attack strategies in differentially private FL. Under the \emph{DP-opt-in attack}, the adversary compromises participating clients but lacks sufficient OS-level privileges to manipulate the local DP perturbation function. Consequently, DP-SGD is applied to the malicious updates, mirroring the behavior of benign clients. This establishes an upper bound on the system's robustness against backdoor injection, as DP noise operates as intended and maximally constrains the backdoor signal. Conversely, the \emph{DP-opt-out attack} assumes the adversary exercises complete control over the compromised devices, matching the capability level of our \ring attack. Rather than crafting an adversarial perturbation, however, this baseline adopts a straightforward but aggressive strategy: it bypasses DP entirely and submits unperturbed malicious updates directly to the server, aiming to maximize attack success rate.

\mybpara{Defender's Capability.}
The server acts as the defender and may deploy any existing in-training mitigations, including DP-SGD, robust aggregation, and representation-level defenses. We do not consider post-training defenses, as our focus is on backdoor injection during the training phase, the setting in which such attacks are most practically relevant.

\vspace{-5pt}
\subsection{Baseline Investigation}\label{sec:investigation}
We begin by investigating the performance of the two baseline attacks. To our knowledge, no prior work has evaluated either strategy against existing backdoor defenses in a DP-FL setting. To address this gap, we assess both baselines against six state-of-the-art defenses: DeepSight~\cite{rieger2022deepsight}, Krum~\cite{blanchard2017machine}, Flame~\cite{nguyen2022flame}, FLShield~\cite{kabir2024flshield}, FreqFed~\cite{fereidooni2023freqfed}, and MESAS~\cite{krauss2023mesas}. Our investigation is organized around two objectives.

\begin{itemize}[leftmargin=*]
\item \emph{IO-1: Detectability of the DP-opt-out attack.} When the attacker bypasses DP, unperturbed malicious updates may exhibit detectable statistical anomalies. We examine whether existing defenses can reliably identify such updates and suppress the backdoor signal.

\item \emph{IO-2: Evasiveness of the DP-opt-in attack.} Although DP noise reduces backdoor efficacy, it may obscure the statistical signature of poisoned updates. We investigate whether this masking effect allows the DP-opt-in attacker to evade SOTA defenses despite reduced attack strength.
\end{itemize}

\textbf{Setup.} We conduct this experiment on the MNIST dataset under an iid data distribution~\cite{wang2022flare,krauss2023mesas,ali2024adversarially,lyu2025two}. Evaluations on more complex datasets under non-iid settings are provided in the main experiment in Section \ref{sec:main_exp}. In each communication round, 30 clients are selected in order from a pool of 120 to participate in local training. Each selected client trains for five epochs over its full local dataset before submitting updates for aggregation. We use the cross-entropy loss and enforce sample-level DP-SGD with privacy budget $\epsilon = 1$, clipping bound $C = 10$, learning rate $\eta = 0.05$, and momentum $0.9$. A fraction $\beta = 0.2$ of participating clients are designated as malicious per round. The poisoned data rate (PDR), i.e., the fraction of each malicious client's local data that contains backdoor samples, is set to $0.5$. We implement a visible-trigger backdoor attack~\cite{gu2017badnets} in which each malicious client superimposes an $8 \times 8$ white square near the top-right corner of its poisoned input images.



\textbf{Metrics.} To evaluate attack performance, we adopt two standard metrics. \emph{Attack success rate} (ASR) measures the fraction of backdoor test samples that the global model misclassifies as the target label. \emph{Accuracy} (Acc) measures the global model's classification performance on a clean test set, capturing the utility cost of any defense.

ASR alone, however, is insufficient to characterize stealthiness. A low ASR may stem from a defense aggressively discarding all updates -- benign and malicious alike -- rather than selectively identifying poisoned ones. To distinguish targeted detection from indiscriminate filtering, we introduce a complementary metric: \emph{retention rate}, defined as the proportion of submitted updates that survive the defense's filtering step. When malicious updates achieve a retention rate comparable to or higher than benign updates, the attack has evaded detection, indicating high stealthiness. Conversely, a disproportionately low retention rate for malicious updates signals effective targeted detection. All results are averaged over five independent experimental runs.




\textbf{Defense Behavior under DP.} To isolate the effect of DP noise on defense behavior, we first examine clean accuracy under varying privacy budgets in the absence of any attack. Figure~\ref{fig_investigation_acc_e} shows that all defenses preserve their relative accuracy ordering compared to the no-DP baseline ($\epsilon = +\infty$), suggesting that DP noise does not fundamentally alter their filtering behavior. In particular, DeepSight and Krum maintain accuracy on par with the no-defense baseline across all privacy budgets. MESAS, FreqFed, FLShield, and Flame exhibit slight degradation. As we show in subsequent experiments, however, this cost is accompanied by stronger backdoor suppression in certain settings, a trade-off that may be acceptable when resilience is the primary objective. 

\begin{figure}[t]
\centering
\includegraphics[scale=0.45]{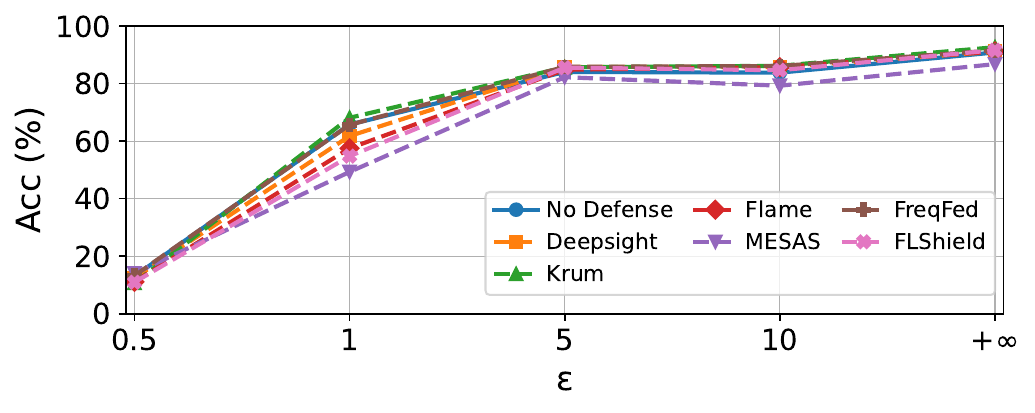}
 \caption{Acc under existing defenses without attack.}
 \label{fig_investigation_acc_e}
\vspace{-10pt}
\end{figure}

\begin{figure*}[htbp]
\centering
\scalebox{1}[0.95]{
\includegraphics[scale=0.35]{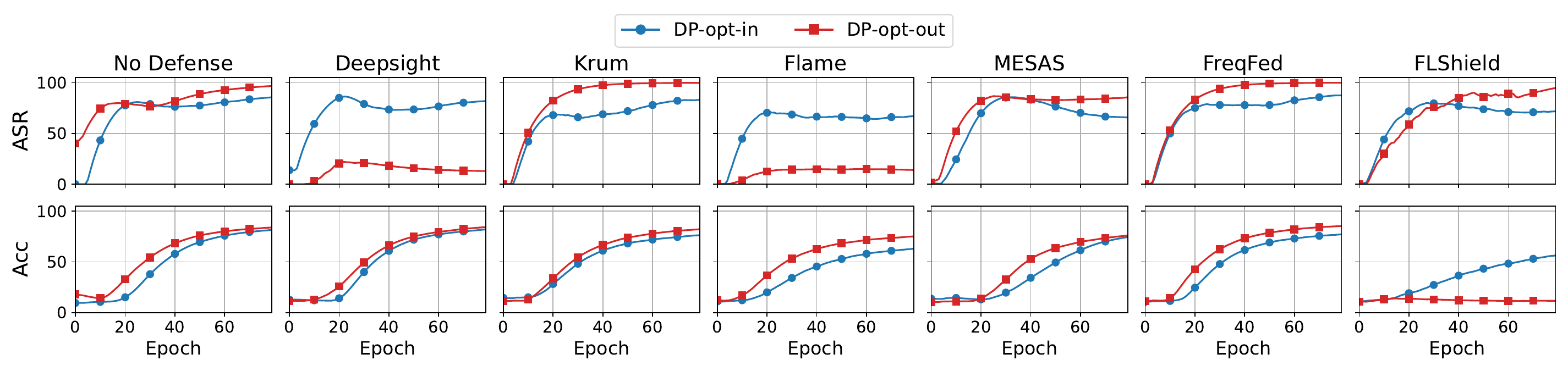}
}
\caption{Performance of existing defenses against DP-opt-in and DP-opt-out attacks.}
 \label{fig_input&output_Backdoor_Main}
\vspace{-2pt}
\end{figure*}

\mybpara{IO-1: Detectability of the DP-opt-out Attack.}
Figure~\ref{fig_input&output_Backdoor_Main} reports ASR and Acc for each defense, and Figure~\ref{fig_retention_input&output} presents the corresponding retention rates. The DP-opt-out attack achieves an ASR close to $100\%$ in the absence of any defense, substantially outperforming the DP-opt-in attack. Among the deployed defenses, DeepSight, Flame, MESAS, and FLShield reduce ASR relative to Krum and FreqFed, though Flame, MESAS, and FLShield also incur a degradation in Acc.

The retention rates in Figure~\ref{fig_retention_input&output} clarify the underlying mechanisms. DeepSight and Flame selectively identify and remove malicious updates, producing low ASR with comparatively modest accuracy costs. FLShield, by contrast, filters both benign and malicious updates without discrimination, suppressing ASR at the expense of main-task accuracy. Krum and FreqFed exhibit the opposite failure mode, i.e., they preferentially discard benign updates while retaining malicious ones, which explains their persistently high ASR.


\begin{tcolorbox}[colback=gray!10, colframe=gray!40, boxrule=0pt, left=1pt, right=1pt, top=1pt, bottom=1pt, arc=0pt]
\mybpara{Takeaways.}  
In practice, a rational adversary will choose the strategy that maximizes attack gain and the DP-opt-out attack does exactly that, achieving near-perfect ASR in an undefended setting. However, our results show that bypassing DP comes at a cost: the absence of DP noise leaves malicious updates statistically distinguishable, allowing several SOTA defenses to detect and suppress them. Attack effectiveness and stealthiness thus remain in fundamental tension for the DP-opt-out strategy.
\end{tcolorbox}

\begin{figure}[htbp]
\centering
\includegraphics[scale=0.3]{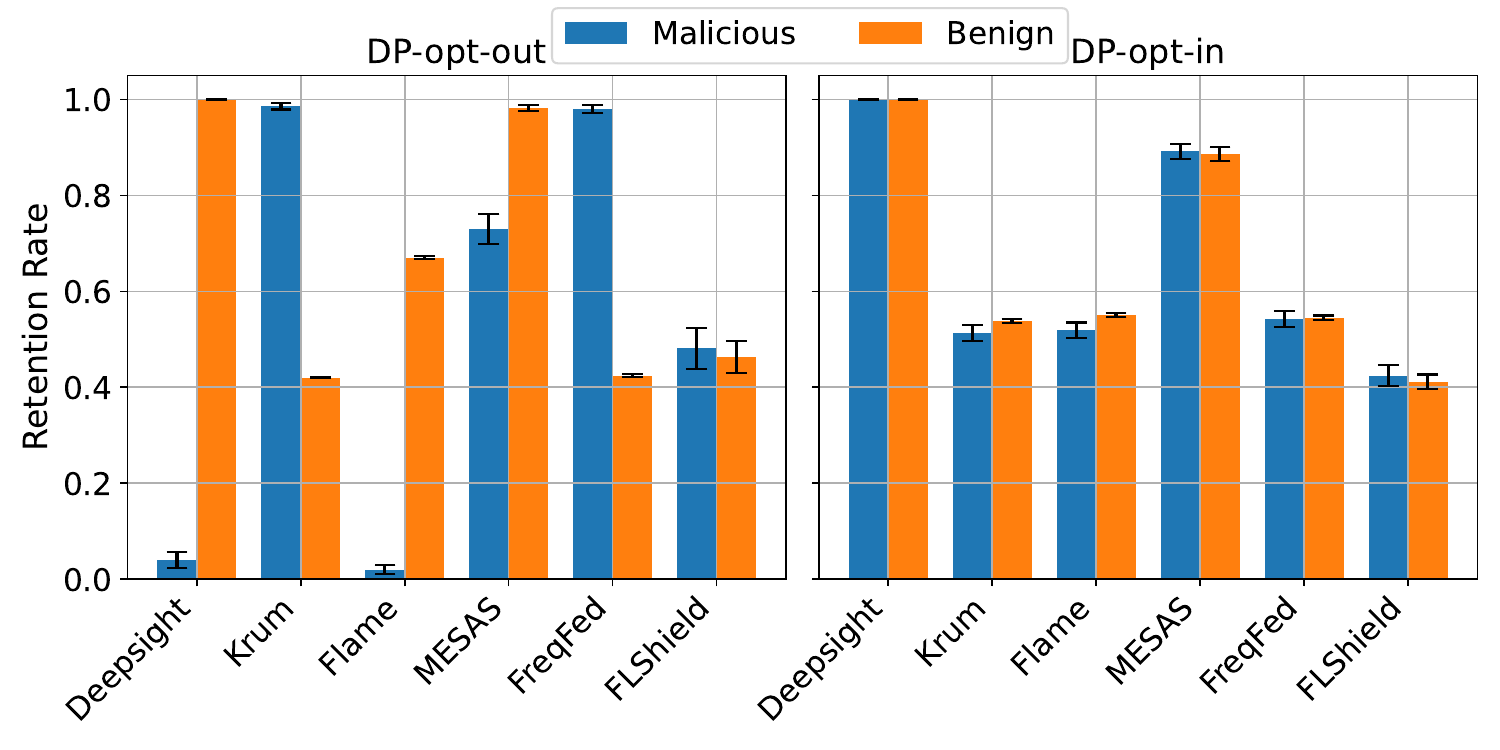}
 \caption{Retention rate of the defenses against DP-opt-in and DP-opt-out attacks with a 95\% confidence interval (CI).}
 \label{fig_retention_input&output}
\end{figure}

\mybpara{IO-2: Evasiveness of the DP-opt-in Attack.} 
As shown in Figure~\ref{fig_input&output_Backdoor_Main}, the DP-opt-in attack yields substantially lower ASR than the DP-opt-out case in the absence of defenses, consistent with the expected suppression of the backdoor signal under DP noise. Notably, however, introducing state-of-the-art defenses produces little additional reduction in ASR relative to the no-defense baseline, which suggests that these mitigations offer limited marginal protection when DP is already applied.

The retention rates in Figure~\ref{fig_retention_input&output} explain why.  Under DP-opt-in, malicious updates are treated almost indiscriminately from benign ones, as both are perturbed by DP-SGD noise. As a result, defenses cannot reliably identify poisoned updates and instead remove updates from both malicious and benign clients without discrimination. Several methods, including Krum, Flame, FreqFed, and FLShield,  exhibit this behavior, incurring Acc degradation without a corresponding reduction in ASR. We provide additional analysis in the Appendix~\ref{app:Geometry}.


\begin{tcolorbox}[colback=gray!10, colframe=gray!40, boxrule=0pt, left=1pt, right=1pt, top=1pt, bottom=1pt, arc=0pt, before skip=6pt, after skip=6pt]
\mybpara{Takeaways.}  
DP perturbation inadvertently undermines existing defenses. While the DP-opt-in attack is inherently weaker in terms of ASR, existing defenses provide little additional suppression, which means that the residual backdoor signal persists largely undetected. This represents a qualitatively different failure mode from the DP-opt-out case, and motivates our design of \ring, that is, an attack that deliberately exploits this masking effect to combine the stealthiness of the DP-opt-in strategy with the effectiveness of the DP-opt-out strategy.
\end{tcolorbox}
\section{\ring Attack}\label{sec:RING}
\begin{figure}[htbp]
\centering
\includegraphics[scale=0.52]{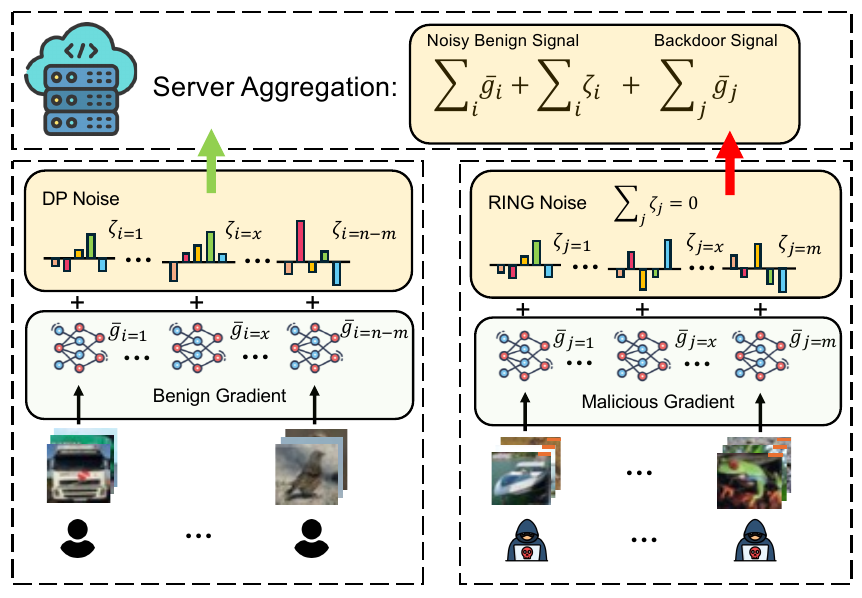}
 \caption{Overview of \ring attack on DP-FL.}
 \label{fig:Overview}
 \vspace{-10pt}
\end{figure}
\subsection{Attack Intuition}

Our investigation reveals two competing objectives that any effective backdoor attacker in DP-FL must jointly satisfy. The first is \textit{effectiveness}: the attacker must recover the backdoor signal suppressed by DP noise, restoring attack strength comparable to the DP-opt-out case and maintaining high ASR after aggregation. The second is \textit{stealthiness}: malicious updates must appear statistically consistent with DP-perturbed benign updates under existing defenses, thereby evading detection.  These objectives are inherently in tension. In other words, improving one typically degrades the other.

We adopt the threat model described in Section~\ref{sec:threat_model}. Figure~\ref{fig:Overview} illustrates the overall framework of \ring. Rather than applying DP noise as benign clients do, malicious clients collaboratively craft adversarial perturbations that jointly satisfy two goals: (1) each individual malicious update is made to resemble a DP-perturbed benign update, improving stealthiness against existing defenses; and (2) the perturbations are coordinated so that they cancel upon aggregation, allowing the underlying backdoor signal to be recovered in the global model.

To formalize this, let $\bar{g}_{j,t}$ denote attacker $j$'s clipped gradient, and let the crafted malicious update be $\tilde{g}_{j,t} = \bar{g}_{j,t} + \zeta_{j,t}$, where $\zeta_{j,t}$ is an attacker-controlled perturbation. Let $\hat{g}_{j,t}$ denote the DP-perturbed gradient as defined in Equation~\eqref{eq:input_backdoor}, $D(\cdot, \cdot)$ a divergence measure between two updates (e.g., Euclidean distance or cosine similarity), and $\mathcal{L}_{\mathrm{bkd}}(\cdot, \mathcal{D}_{adv,j})$ the backdoor loss evaluated on attacker $j$'s held-out validation set $\mathcal{D}_{adv,j}$. The per-attacker objective is $\min_{\zeta_{j,t}} \lambda\,D\big(\tilde{g}_{j,t},\hat g_{j,t}\big) +\xi \mathcal{L}_{\mathrm{bkd}}\big(G_t-\tilde{g}_{j,t}, \mathcal{D}_{adv,j}\big)$. The first term encourages $\tilde{g}_{j,t}$ to be close to a DP-perturbed update $\hat{g}_{j,t}$, improving stealthiness; the second term minimizes backdoor loss, improving effectiveness. As our baseline investigation shows, jointly minimizing these two opposing terms is nontrivial without an additional structural constraint.

Inspired by secret sharing~\cite{delfs2002introduction} in cryptography, where a secret can be recovered from the aggregation of individually randomized shares, we treat the original malicious gradients as the ``secret'' and design the perturbations $\{\zeta_{j,t}\}$ to serve as obfuscating shares. Specifically, let $\mathcal{M}_t$ denote the set of $m$ malicious clients selected in round $t$. We impose the constraint $\sum_{j \in \mathcal{M}_t} \zeta_{j,t} = 0$, ensuring that the perturbations cancel upon aggregation and the cumulative backdoor signal is preserved. The joint optimization over all malicious clients is then:
\begin{align}\nonumber\label{eq:group_attacker_obj}
    \min_{\{\zeta_{j,t}\}_{j\in\mathcal{M}_t}}
&\sum_{j\in\mathcal{M}_t}\big[\lambda\,D\big(\tilde{g}_{j,t},\,\hat g_{j,t}\big) \\\nonumber
&+\xi\,\mathcal{L}_{\mathrm{bkd}}\!\big(G_t-\tilde{g}_{j,t}, \mathcal{D}_{adv,j}\big)\big]\\ 
\quad\text{s.t.}\quad
&\sum_{j\in\mathcal{M}_t} \zeta_{j,t}=0.
\end{align}
Rather than applying a general optimizer to solve \eqref{eq:group_attacker_obj} directly, we present a practical and efficient solution in the next section.

\subsection{Attack Details}\label{sec:attack_details}

We now present the concrete construction of \ring. In each round $t$, the $m$ malicious clients in $\mathcal{M}_t$ are partitioned into $g$ disjoint subgroups ${\mathcal{G}_1, \ldots, \mathcal{G}_g}$, where subgroup $\mathcal{G}_\ell$ contains $m_\ell$ clients. Within each subgroup, the members coordinate their perturbations such that the noise terms cancel upon summation, i.e., \(\sum_{j\in\mathcal{G}_\ell}\zeta_{j,t}=0\). To satisfy this constraint, each client $j \in \mathcal{G}_\ell$ draws an independent sample $z_j \sim \mathcal{N}(0, \sigma^2 I_d)$ and computes its perturbation as
\begin{equation}\label{eq:ring}
\zeta_{j,t} = z_j - \frac{1}{m_\ell} \sum_{k \in \mathcal{G}_\ell} z_k,
\end{equation}
which ensures exact cancellation by construction. The crafted update submitted by client $j$ is then $\tilde{g}_{j,t} = \bar{g}_{j,t} + \zeta_{j,t}$, as defined in Section~\ref{sec:RING}. This construction is agnostic to how subgroups are formed. The only practical requirement is that all members of a subgroup are selected by the server within the same communication round; once selected, the cancellation in \eqref{eq:ring} holds exactly upon aggregation. Theorem~\ref{theorem:noise_component} characterizes the Gaussian marginal of each perturbation $\zeta_{j,t}$.

\begin{theorem}\label{theorem:noise_component}
Suppose the noise term $\zeta_j = z_j - \frac{1}{m_l}\sum_{k\in\mathcal{G}_l} z_k$ and $z_k \sim \mathcal{N}(0,\sigma^2 I_d)$ is an \textit{i.i.d.} variable. Then we have $$\zeta_j \sim \mathcal{N}\big(0,\frac{m_l-1}{m_l}\sigma^2 I_d\big).$$
\end{theorem}
\begin{proof}
    See Appendix~\ref{app:Proof_1}.
\end{proof}
When the subgroup size is $m_\ell$, the variance of the perturbation $\zeta_{j,t}$ is $\frac{m_\ell - 1}{m_\ell}\sigma^2 I_d$, compared to $\sigma^2 I_d$ for the DP-SGD noise. As $m_\ell$ increases, the marginal perturbation variance approaches that of DP-SGD noise. The effect of subgroup size on attack performance is examined empirically in the ablation study (Section~\ref{sec:ablation}).

Equation~\eqref{eq:ring} constitutes a closed-form solution to a relaxed form of the joint objective in \eqref{eq:group_attacker_obj}. By construction, $\zeta_{j,t}$ simultaneously serves two roles: it gives each perturbation a Gaussian marginal resembling DP noise, minimizing the distance term $D(\tilde{g}_{j,t}, \hat{g}_{j,t})$; and it ensures that the perturbations cancel exactly upon aggregation, allowing the underlying backdoor signal to be recovered in the global model and thus minimizing the backdoor loss $\mathcal{L}_{\mathrm{bkd}}$. Together, these two properties allow \ring to match the stealthiness of the DP-opt-in attack while preserving the effectiveness of the DP-opt-out attack.

The derivation above assumes equal aggregation weights $w_i = 1$ in \eqref{eq:weighted_fedavg}. In practice, weight variations, arising from robust aggregation mechanisms or heterogeneous client dataset sizes, may perturb the exact cancellation in \eqref{eq:group_attacker_obj}. Our empirical evaluation in Section~\ref{sec:results_ring} shows that \ring remains robust to moderate weight deviations and sustain high ASR across a range of realistic deployment conditions.



\section{Theoretical Analysis}\label{sec:performance}

Our experiments (Figure~\ref{fig_retention_input&output}) reveal that existing defenses rarely eliminate all malicious updates and partial removal is the norm in practice. This disrupts the noise cancellation that underpins \ring's effectiveness.  In this section, we analyze three aspects of \ring: the effect of partial aggregation on attack effectiveness (Section~\ref{sec:effectiveness}), the conditions under which crafted updates evade detection (Section~\ref{sec:stealthy}), and the generality of \ring across different backdoor techniques (Section~\ref{sec:generalization}).

\subsection{Partial Aggregation}\label{sec:effectiveness}
Attack effectiveness depends directly on the level of residual noise after partial removal: as more malicious updates are filtered, cancellation becomes less complete and the recovered backdoor signal weakens. We model this by assuming the server retains each malicious update independently with probability $f \in (0,1)$, where $f = \mathbb{E}[s/m]$, $s = |\mathcal{S}_t|$ denotes the number of retained updates, and $|\mathcal{S}_t| \sim \mathrm{Binomial}(m, f)$. Theorem~\ref{thm:effectiveness} quantifies how partial removal affects the aggregate noise error under varying $f$ and subgroup configurations.

\begin{theorem}\label{thm:effectiveness}
    Suppose $m$ malicious clients are divided into $g = m/m_l$ disjoint subgroups of size $m_l$, and each client is retained independently with probability $f$. For the \ring attack, the expected squared norm of the aggregate noise error after a defense is
    \begin{align}
        \mathbb{E}_{\ring}\bigl[\|\mathrm{Err}(f)\|^2\bigr] \approx \frac{d\sigma^2}{m}\cdot\frac{m_l-1}{m_l}\cdot \frac{1-f}{f},
    \end{align}
    where $d$ is the dimension of the noise vectors.
\end{theorem}

\begin{proof}
    See Appendix~\ref{app:Proof_2}
\end{proof}
As Theorem~\ref{thm:effectiveness} shows, the noise error grows as $f$ decreases and vanishes as $f \to 1$. The result also reveals a subgroup-size trade-off: for a fixed $f$, larger subgroups $m_l$ increase aggregate noise variance, weakening cancellation and degrading attack performance. The attacker can therefore improve resilience against partial removal by using smaller subgroups.


\begin{theorem}\label{thm:err_compare}
  Under identical noise variance $\sigma^2$ and dimension $d$, the expected squared errors for \ring and the DP-opt-in attack satisfy
    \begin{align}
        \frac{\mathbb{E}_{\ring}\bigl[\|\mathrm{Err}(f)\|^2\bigr]}{\mathbb{E}_{\text{DP-opt-in}}\bigl[\|\mathrm{Err}(f)\|^2\bigr]}
        \approx \frac{m_l - 1}{m_l}(1 - f),
    \end{align}
    where $m_l$ is the subgroup size in \ring.
\end{theorem}

\begin{proof}
    See Appendix~\ref{app:Proof_3}
\end{proof}
Since $\frac{m_l-1}{m_l}(1-f) < 1$ for any $m_l > 1$ and $f \in (0,1)$, \ring consistently produces less residual noise than the DP-opt-in attack under any partial-removal regime. This is a meaningful distinction from naive secret sharing, where removing any subset of shares does not reduce the residual noise magnitude. Here, partial retention proportionally reduces the noise error, yielding a stronger backdoor signal than DP-opt-in regardless of the retention rate. Section~\ref{sec:mitigation} further evaluates robustness under random client dropping, confirming that \ring maintains high ASR even under substantial removal rates.


\subsection{Stealthiness}\label{sec:stealthy} 

Most existing defenses detect backdoor attacks by measuring similarity or statistical deviation among client updates~\cite{fereidooni2023freqfed,kabir2024flshield,krauss2023mesas,nguyen2022flame,blanchard2017machine,rieger2022deepsight}. A model update can be decomposed into a clipped gradient $\bar{g}$ and a noise term $\zeta$. Therefore, the update of a malicious client $j$ is
\begin{align}\nonumber
    \tilde{g}_j = \bar{g}_j + \zeta_j, \ \zeta_j \sim \mathcal{N}(0,\frac{m_l-1}{m_l}\sigma^2I_d)
\end{align}
while the update of a benign client $i$ is
\begin{align}\nonumber
    \tilde{g}_i = \bar{g}_i + \zeta_i, \ \zeta_i \sim \mathcal{N}(0,\sigma^2I_d).
\end{align}

Under a similarity-based defense, the server computes $D(\tilde{g}_i, \tilde{g}_j) = D(\bar{g}_i + \zeta_i, \bar{g}_j + \zeta_j)$. This quantity reflects two distinct sources of discrepancy: $D(\bar{g}_i, \bar{g}_j)$, which captures differences in the underlying data distributions, and $D(\zeta_i, \zeta_j)$, which reflects the statistical difference between DP noise and the \ring perturbation. The first term is data-dependent and outside the attacker's control; in non-iid settings it can be large, in which case malicious updates may remain detectable regardless of how well the noise is designed. \ring does not remove this source of detectability. When $D(\bar{g}_i, \bar{g}_j)$ is small, however, detection is primarily governed by $D(\zeta_i, \zeta_j)$, where the attacker has direct influence.

The constraint $\sum_{j \in \mathcal{G}_\ell} \zeta_{j,t} = 0$ in Equation~\eqref{eq:ring} introduces negative correlations among the perturbations within each subgroup. Marginal similarity does not imply joint indistinguishability. In the minimal case of $m_l = 2$, the two noise vectors satisfy $\zeta_a = -\zeta_b$ almost surely, which makes the pair geometrically conspicuous to similarity-based detectors. As $m_l$ increases, both the variance gap relative to DP noise (a factor of $\tfrac{m_l - 1}{m_l}$) and the magnitude of negative correlations decrease. Empirically, a moderate group size such as $m_l = 4$ can substantially reduce the detectability contributed by $D(\zeta_i, \zeta_j)$, as we show in Section~\ref{sec:ablation}. Even so, a large $D(\bar{g}_i, \bar{g}_j)$ may still cause false positives regardless of the perturbation design, and we observe this behavior in several non-iid settings in Section~\ref{sec:results_ring}.

\vspace{2pt}
\textbf{A Large or Small $m_l$?} The analysis above reveals a trade-off: increasing $m_l$ improves perturbation-side stealthiness at the cost of higher aggregate noise variance under partial removal (Theorem~\ref{thm:effectiveness}), while a large data-dependent gap $D(\bar{g}_i, \bar{g}_j)$ limits the overall benefit. In practice, the attacker should select $m_l$ to balance these competing factors. Notably, our experiments in Sections~\ref{sec:results_ring} and~\ref{sec:ablation} show that even when all malicious clients are placed in a single group, \ring still achieves high ASR.


\subsection{Generalization}\label{sec:generalization} 
The decomposition above also clarifies the generality of \ring. \ring operates as an adversarial perturbation mechanism on the noise term $\zeta_j$, and leaves the clipped task gradient $\bar{g}_j$ produced by the attacker's chosen backdoor objective intact. As a result, \ring is agnostic to the source of $\bar{g}_j$: different backdoor techniques affect only the task-gradient component, while \ring governs the perturbation component responsible for stealthiness and aggregation-level cancellation. In our main evaluation, we demonstrate this with visible-trigger image backdoors~\cite{gu2017badnets} and edge-case text backdoors~\cite{wang2020attack}, both of which remain effective against state-of-the-art defenses. Section~\ref{sec:results_ring} further validates this generality on more advanced baselines, including DBA~\cite{xie2019dba} and the Neurotoxin attack~\cite{pmlr-v162-zhang22w}.


\begin{figure*}[htbp]
\centering
\vspace{-5pt}
\scalebox{1}[0.95]{
\includegraphics[scale=0.335]{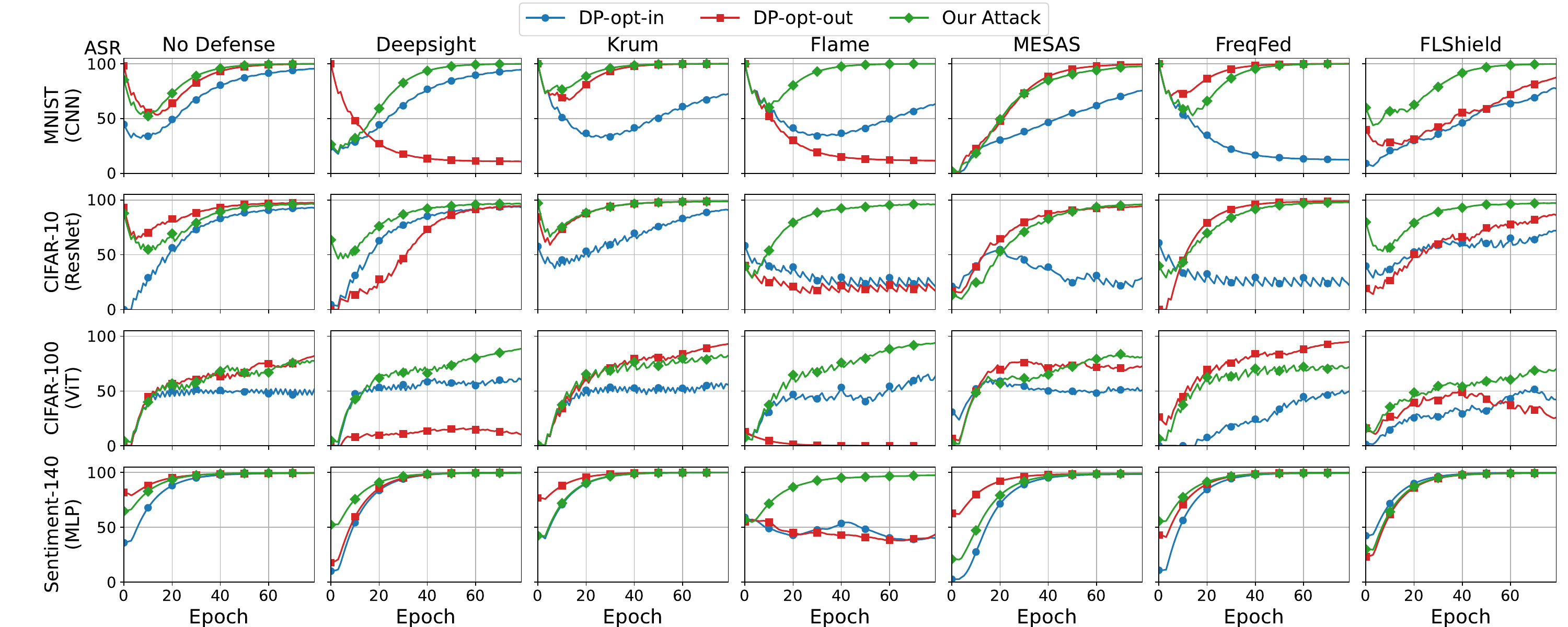}
}
 \caption{ASR of DP-opt-in, DP-opt-out, and \ring attacks in the prob non-iid setting.}
 \label{fig:ASR_e=5_m=6}
 \vspace{-8pt}
\end{figure*}

\begin{figure*}[htbp]
\centering
\vspace{-5pt}
\scalebox{1}[0.9]{
\includegraphics[scale=0.304]{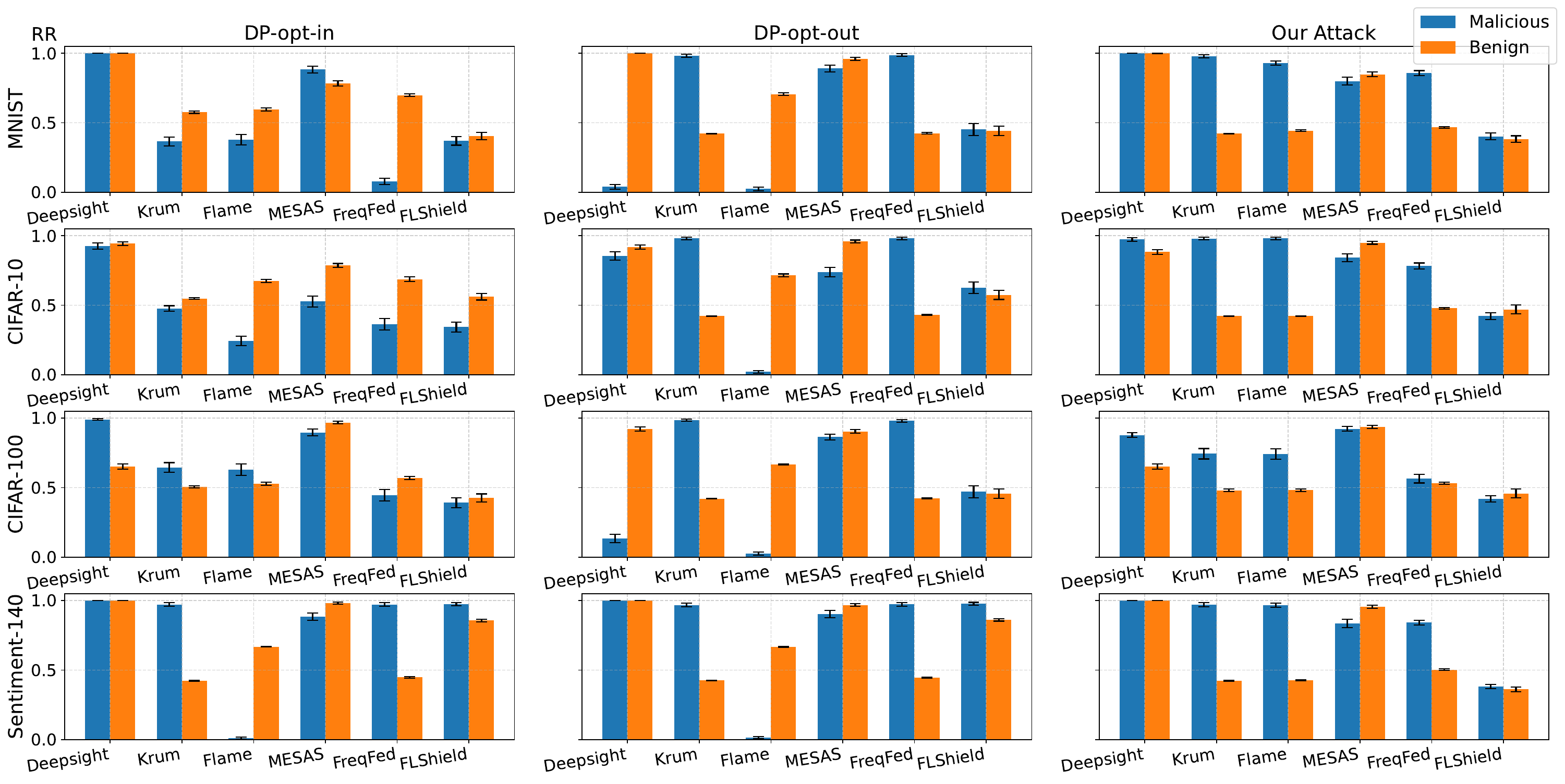}
}
 \caption{Retention rate of DP-opt-in, DP-opt-out and \ring attacks in the prob non-iid setting.}
 \label{fig:RR_e=5_m=6}
 \vspace{-10pt}
\end{figure*}

\section{Experiments}\label{sec:main_exp} 
\subsection{Experimental Setting}\label{sec:experimental_setting}


\textbf{Datasets and Models.}
We evaluate our attack on four datasets used in prior work~\cite{fereidooni2023freqfed,rieger2022crowdguard,krauss2023mesas,xie2023unraveling}: MNIST, CIFAR-10, and CIFAR-100 for image tasks, and Sentiment-140~\cite{go2009twitter} for text. MNIST and CIFAR-10 each contain 10 classes, CIFAR-100 contains 100 classes, and Sentiment-140 comprises 1.6 million tweets annotated for binary sentiment; we randomly sample $5\%$ of Sentiment-140 for our experiments. We employ CNN, ResNet-18, ViT, and fully connected neural networks, respectively, for the four datasets. Details about datasets and model architectures are summarized in Table~\ref{tab:models} (Appendix~\ref{app_sec:datasets_models}). 
We evaluate three non-iid data distributions with default parameters:
\begin{enumerate}
\item \textbf{Probability-based (prob).} Each client $i$ has a dominant class $c_i$ drawn uniformly from $K$ classes, with sampling probabilities $p_i(c_i) = \tau$ and $p_i(k \neq c_i) = \frac{1-\tau}{K-1}$, where $\tau = 0.5$.
\item \textbf{Dirichlet-based (dir).} Each client receives a variable number of samples per label, with potentially zero samples for some labels, drawn according to a Dirichlet$(\alpha = 0.5)$ distribution over class proportions.
\item \textbf{Quantity-based (qty).} Each client is assigned data from a fixed number of classes sampled uniformly from the global dataset: 2 classes for MNIST and CIFAR-10, 20 for CIFAR-100, and 1 for Sentiment-140.
\end{enumerate}



\begin{table}[bthp]
\vspace{-5pt}
\setlength\tabcolsep{8pt}
\small
\centering
\caption{Default Parameters.}
\label{tab:default_param}
\resizebox{\columnwidth}{!}{%
\begin{tabular}{c|c}
\hline
\textbf{Parameters} & \textbf{Default Values} \\
\hline
\#clients $n$ & 120 \\
\#clients per round $r$ & 0.25 \\
\%malicious clients $\beta$ & 0.2 \\
PDR & 0.5 \\
Privacy budget $\epsilon$ & 5, 20 (CIFAR-10/100) \\
Clipping bound $C$ & 10 \\
Learning rate $\eta$ & 0.05, 0.01 (CIFAR-10/100), 0.1 (Sentiment-140) \\
\#local epochs & 5, 10 (CIFAR-10/100) \\
\hline
\end{tabular}
}
\end{table}

\textbf{Setup.} Unless otherwise specified, all experiments follow the configurations in Section~\ref{sec:investigation}, with default parameters summarized in Table~\ref{tab:default_param}.  We empirically select $\epsilon=20$ for CIFAR-10/100 to preserve model utility, although this represents a relatively weak privacy regime. This utility-driven choice can disadvantage \ring by reducing the DP-noise cover available for evasion, while very small budgets can cause utility collapse (Section~\ref{sec:impact_epsilon}). Appendix~\ref{app:privacy_accounting} reports the privacy accounting parameters. In the ablation study in Section~\ref{sec:ablation}, we further assess attack performance across various key parameters and conditions. 

For image datasets, we use the visible-trigger backdoor described in Section~\ref{sec:backdoor_background}. For Sentiment-140, we adopt the edge-case text backdoor~\cite{wang2020attack}. We report the same metrics as in Section~\ref{sec:investigation}: ASR, Acc, and retention rate (RR), where RR is reported with its 95\% confidence interval.

For Krum, which requires an estimate of the number of malicious clients, we supply the true value to reflect its best-case detection performance. We adopt the weighted FedAvg as in Section~\ref{sec:Fed_Learn}. Since \ring performs best when $w_i=1$, this configuration places our attack at a disadvantage relative to its optimal setting; our results nonetheless show that \ring maintains high ASR throughout.


All experiments are implemented in Python with PyTorch and run on two GPU servers: one with Intel Xeon Platinum 8468 CPUs, 8 NVIDIA H100 GPUs, and 1~TB of memory; the other with AMD EPYC 9554 CPUs, 3 NVIDIA RTX~6000~Ada GPUs, and 768~GB of memory. Datasets, pre-processing details, and source code are available at \url{https://anonymous.4open.science/r/RING-191A/}.



\vspace{-5pt}
\subsection{Main Results}\label{sec:results_ring}
As discussed in Section~\ref{sec:performance}, subgroup size $m_l$ governs the trade-off between stealthiness and effectiveness: larger subgroups reduce per-client detectability but dilute the per-round backdoor signal. Since the DP-opt-in attack — which relies solely on DP noise for masking — remains partially detectable under certain tasks and non-iid settings, we set $m_l = m$, placing all malicious clients in a single group to maximize stealthiness.

Since main-task accuracy under all attacks is close and consistent across datasets, we report it separately in Figure~\ref{fig:Acc_e=5_m=6} (Appendix~\ref{app_sec:results_backdoor}) and focus on ASR in Figure~\ref{fig:ASR_e=5_m=6}. These results confirm that \ring does not degrade main-task performance, consistent with standard backdoor attack behavior. Figure~\ref{fig:RR_e=5_m=6} reports the corresponding retention rates for stealthiness evaluation. As image task results are consistent across non-iid settings, we focus on the prob non-iid configuration and defer the remaining results (Figures~\ref{apx_fig:ASR_e=5_m=6_qty}--\ref{apx_fig:RR_e=5_m=6_dir}) to Appendix~\ref{app_sec:results_backdoor}. Across all datasets, \ring substantially outperforms both DP-opt-in and DP-opt-out baselines in ASR and evades existing defenses in most cases. We have the following key observations.

\begin{table*}[htbp]
\centering
\scriptsize
 \vspace{-2pt}
 \setlength{\tabcolsep}{3.2pt}
\caption{ASR, Acc (\%) and RR (malicious/benign,\%) of \ring, DP-opt-in, and DP-opt-out attacks with varying $m$.}
\label{tab:ASR_ablation_m}
\resizebox{\textwidth}{!}{%
\begin{tabular}{c|ccccccccccccccc}
\toprule
\multirow{2}{*}{Dataset} &\multirow{2}{*}{}  & \multirow{2}{*}{m($\beta$)} & {No Defense}& \multicolumn{2}{c}{Deepsight} & \multicolumn{2}{c}{Krum} & \multicolumn{2}{c}{Flame} & \multicolumn{2}{c}{MESAS} & \multicolumn{2}{c}{FreqFed} & \multicolumn{2}{c}{FLShield} \\
& & & ASR / Acc& ASR / Acc&RR &ASR / Acc&RR &ASR / Acc&RR &ASR / Acc&RR&ASR / Acc&RR &ASR / Acc&RR\\
\midrule
\multirow{9}{*}{MNIST} 
& \multirow{3}{*}{DP-opt-in} & 2 (6.7\%) & 40.89 / 89.44 & 30.14 / 90.15 & 100 / 100 & 20.15 / 89.27 & 41 / 82 & 13.54 / 87.23 & 18 / 57 & 13.38 / 83.27 & 47 / 88 & 11.41 / 88.70 & 6 / 58 & 18.96 / 88.72 & 34 / 41 \\
& & 4 (13.3\%) & 82.31 / 90.70 & 75.72 / 90.59 & 100 / 100 & 54.90 / 90.20 & 43 / 70 & 30.37 / 90.34 & 29 / 58 & 64.21 / 84.97 & 87 / 78 & 13.50 / 90.50 & 12 / 62 & 83.25 / 87.46 & 39 / 38 \\
& & 6 (20.0\%) & 96.56 / 89.14 & 95.61 / 88.45 & 100 / 100 & 77.49 / 89.68 & 36 / 57 & 69.25 / 87.85 & 37 / 59 & 79.60 / 83.33 & 88 / 78 & 12.12 / 89.40 & 7 / 69 & 82.19 / 87.20 & 37 / 40 \\\cline{2-16} 
& \multirow{3}{*}{DP-opt-out} & 2 & 88.91 / 92.16 & 68.18 / 90.75 & 73 / 98 & 95.28 / 89.77 & 99 / 78 & 11.39 / 87.90 & 2 / 57 & 14.54 / 58.26 & 6 / 98 & 99.61 / 89.60 & 97 / 50 & 22.53 / 89.01 & 45 / 43 \\
& & 4 & 99.85 / 90.48 & 57.87 / 90.57 & 34 / 92 & 99.99 / 90.13 & 98 / 61 & 11.21 / 88.04 & 2 / 63 & 82.18 / 83.49 & 53 / 95 & 99.99 / 89.74 & 97 / 46 & 80.81 / 86.57 & 44 / 43 \\
& & 6 & 99.95 / 90.86 & 10.66 / 91.19 & 4 / 100 & 99.99 / 91.04 & 98 / 42 & 11.36 / 89.50 & 2 / 70 & 99.80 / 88.11 & 89 / 95 & 100.0 / 92.12 & 98 / 42 & 91.89 / 88.19 & 45 / 44 \\\cline{2-16} 
& \multirow{3}{*}{Our Attack} & 2 & 86.92 / 90.26 & \cellcolor{gray!15}\textbf{90.49} / \textbf{91.06} & 94 / 99 & 94.11 / \cellcolor{gray!15}\textbf{91.99} & 99 / 78 & \cellcolor{gray!15}\textbf{83.10} / \textbf{90.03} & 44 / 55 & 3.88 / 52.58 & 5 / 97 & \cellcolor{gray!15}\textbf{99.95} / \textbf{91.73} & 97 / 50 & \cellcolor{gray!15}\textbf{90.54} / \textbf{90.38} & 46 / 44 \\
& & 4 & 99.47 / 88.78 & \cellcolor{gray!15}\textbf{99.68} / \textbf{91.99} & 96 / 99 & \cellcolor{gray!15}\textbf{99.99} / 89.71 & 98 / 61 & \cellcolor{gray!15}\textbf{99.97} / 89.20 & 99 / 46 & \cellcolor{gray!15}\textbf{93.46} / \textbf{87.09} & 60 / 89 & \cellcolor{gray!15}\textbf{99.99} / \textbf{91.34} & 94 / 50 & \cellcolor{gray!15}\textbf{99.22} / \textbf{89.67} & 39 / 42 \\
& & 6 & \cellcolor{gray!15}\textbf{99.98} / \textbf{91.70} & \cellcolor{gray!15}\textbf{99.99} / \textbf{91.37} & 100 / 99 & \cellcolor{gray!15}\textbf{100.0} / \textbf{91.52} & 97 / 42 & \cellcolor{gray!15}\textbf{99.99} / \textbf{90.42} & 92 / 44 & 98.12 / 83.85 & 79 / 84 & \cellcolor{gray!15}\textbf{100.0} / 91.84 & 85 / 46 & \cellcolor{gray!15}\textbf{99.96} / \textbf{89.68} & 40 / 38 \\
\midrule
\multirow{8}{*}{\shortstack{Sentiment\\-140}}
& \multirow{3}{*}{DP-opt-in} & 2 & 93.18 / 65.11 & 95.75 / 59.67 & 100 / 100 & 96.46 / 65.01 & 98 / 78 & 35.36 / 58.22 & 1 / 57 & 19.96 / 57.89 & 7 / 98 & 99.19 / 60.96 & 97 / 52 & 97.26 / 62.89 & 97 / 97 \\
& & 4 & 99.61 / 58.12 & 98.45 / 54.45 & 100 / 100 & 98.57 / 59.16 & 97 / 61 & 54.05 / 55.37 & 1 / 61 & 99.25 / 60.30 & 86 / 97 & 99.16 / 57.83 & 97 / 48 & 98.71 / 60.16 & 98 / 86 \\
& & 6 & 99.62 / 59.25 & 99.96 / 65.36 & 100 / 100 & 100.0 / 63.83 & 97 / 42 & 41.65 / 55.20 & 1 / 66 & 98.48 / 62.64 & 88 / 98 & 99.75 / 60.54 & 97 / 44 & 99.76 / 58.71 & 97 / 85\\\cline{2-16} 
& \multirow{3}{*}{DP-opt-out} & 2 & 95.07 / 57.95 & 94.69 / 64.80 & 100 / 100 & 94.20 / 61.16 & 97 / 78 & 39.38 / 55.30 & 1 / 57 & 37.94 / 56.98 & 4 / 98 & 97.93 / 58.41 & 97 / 52 & 52.89 / 55.81 & 1 / 8 \\
& & 4 & 97.49 / 56.50 & 99.22 / 59.58 & 100 / 100 & 99.87 / 59.09 & 97 / 61 & 39.89 / 57.95 & 1 / 61 & 99.22 / 58.96 & 74 / 96 & 100.0 / 57.67 & 97 / 48 & 99.24 / 59.40 & 97 / 86\\
& & 6 & 99.39 / 59.08 & 99.98 / 59.02 & 100 / 100 & 100.0 / 55.32 & 96 / 42 & 47.73 / 57.03 & 1 / 66 & 98.84 / 63.48 & 90 / 96 & 99.92 / 57.34 & 97 / 44 & 99.58 / 58.83 & 97 / 85 \\\cline{2-16} 
& \multirow{3}{*}{Our Attack} & 2 & \cellcolor{gray!15}\textbf{95.55} / 55.86 & \cellcolor{gray!15}93.37 / \textbf{66.62} & 100 / 100 & \cellcolor{gray!15}\textbf{96.74} / 59.67 & 98 / 78 & 24.86 / 57.28 & 1 / 57 & \cellcolor{gray!15}\textbf{52.58} / \textbf{62.58} & 5 / 98 & 98.37 / 59.25 & 97 / 52 & \cellcolor{gray!15}90.65 / \textbf{63.77} & 49 / 49 \\
& & 4 & 99.59 / 56.98 & 98.36 / 59.50 & 100 / 100 & \cellcolor{gray!15}98.81 / \textbf{63.59} & 97 / 61 & \cellcolor{gray!15}\textbf{94.88} / 53.37 & 97 / 46 & \cellcolor{gray!15}97.11 / \textbf{66.36} & 90 / 96 & \cellcolor{gray!15}99.62 / \textbf{64.36} & 91 / 50 & \cellcolor{gray!15}97.86 / \textbf{63.47} & 43 / 49 \\
& & 6 & \cellcolor{gray!15}99.61 / \textbf{62.54} & 99.61 / 61.65 & 100 / 100 & \cellcolor{gray!15}\textbf{100.0} / 55.01 & 96 / 42 & \cellcolor{gray!15}\textbf{97.55} / 54.77 & 96 / 42 & \cellcolor{gray!15}\textbf{99.08} / 58.58 & 83 / 95 & 99.50 / 56.77 & 84 / 50 & \cellcolor{gray!15}99.47 / \textbf{60.11} & 38 / 36 \\
\bottomrule
\end{tabular}
}
\vspace{-3pt}
\end{table*}

\textbf{Attack Performance without Defense.} As shown in Figure~\ref{fig:ASR_e=5_m=6}, \ring achieves ASR comparable to the DP-opt-out attack, which serves as a practical upper bound on ASR in DP-protected FL without defenses. This confirms that \ring successfully recovers the backdoor signal from perturbed updates, consistent with the objective in Equation~\ref{eq:group_attacker_obj}. Under the relatively high privacy budget of $\epsilon = 5$, the DP-opt-in attack reaches similar ASR; however, as shown in Section~\ref{sec:ablation}, its performance degrades substantially as $\epsilon$ decreases.


\textbf{Attack Performance with Defense.} \ring demonstrates strong resilience against state-of-the-art defenses. As shown in Figures~\ref{fig:ASR_e=5_m=6} and~\ref{fig:RR_e=5_m=6}, \ring consistently achieves an average ASR of $90.3 \pm 13.8\%$ at the final training round, while the DP-opt-in and DP-opt-out baselines reach only $57.5 \pm 39.3\%$ and $51.9 \pm 46.6\%$, respectively, with substantially higher variability across datasets and defenses. The retention rates corroborate this: none of the defenses succeeds in selectively identifying \ring's malicious updates, whereas both baselines are detectable under most settings regardless of whether DP perturbation is applied. DP-opt-in becomes detectable mainly because it is relatively sensitive to underlying data distributions as analyzed in Section~\ref{sec:stealthy}.

\textbf{Defense Performance.} Figure~\ref{fig:RR_e=5_m=6} shows that Krum, Flame, and FreqFed inadvertently favor \ring by retaining more malicious updates than benign ones. Krum and FreqFed are more effective at detecting the DP-opt-in attack than the DP-opt-out attack, whereas Flame suppresses both baselines in most settings. DeepSight, MESAS, and FLShield show limited ability to distinguish malicious from benign updates across most configurations.


\textbf{Image vs. Text Datasets.} As shown in Figure~\ref{fig:RR_e=5_m=6}, all defenses behave consistently against \ring and the DP-opt-out attack across both image and text datasets. Against the DP-opt-in attack, Krum, FreqFed, and FLShield show stronger detection performance on image datasets than on Sentiment-140.


\textbf{Non-iid Settings.} Data distribution has a pronounced effect on defense behavior. On Sentiment-140, the performance of Krum, Flame, FreqFed, and FLShield against \ring differs substantially between the prob non-iid setting (Figure~\ref{fig:RR_e=5_m=6}) and the dir non-iid setting (Figure~\ref{apx_fig:RR_e=5_m=6_dir}, Appendix~\ref{app_sec:results_backdoor}). Flame, for instance, inadvertently favors \ring under the prob setting but detects malicious updates effectively under the dir setting. The same pattern holds for the DP-opt-in and DP-opt-out baselines.



\begin{table*}[htbp]
\centering
\scriptsize
\setlength{\tabcolsep}{3.7pt}
\caption{ASR, Acc (\%) and RR (malicious/benign,\%) of \ring, DP-opt-in, and DP-opt-out attacks with varying $\epsilon$.}.
\label{tab:ASR_ablation_e}
\resizebox{\textwidth}{!}{%
\begin{tabular}{c|ccccccccccccccc}
\toprule
\multirow{2}{*}{Dataset} &\multirow{2}{*}{}  & \multirow{2}{*}{$\epsilon$} & {No Defense}& \multicolumn{2}{c}{Deepsight} & \multicolumn{2}{c}{Krum} & \multicolumn{2}{c}{Flame} & \multicolumn{2}{c}{MESAS} & \multicolumn{2}{c}{FreqFed} & \multicolumn{2}{c}{FLShield} \\
& & & ASR / Acc & ASR / Acc &RR &ASR / Acc &RR &ASR / Acc &RR &ASR / Acc &RR &ASR / Acc &RR &ASR / Acc &RR\\
\midrule
\multirow{9}{*}{MNIST} 
& \multirow{3}{*}{DP-opt-in} & 1 & 77.27 / 79.32 & 83.51 / 78.92 & 100 / 100  & 41.85 / 64.75 & 41 / 56 & 65.16 / 55.39 & 41 / 56 & 71.23 / 68.53 & 90 / 79 & 60.68 / 64.30 & 42 / 56  & 90.61 / 11.38 & 38 / 37 \\
& & 5 & 96.56 / 89.14 & 95.61 / 88.45 & 100 / 100  & 77.49 / 89.68 & 36 / 57 & 69.25 / 87.85 & 37 / 59 & 79.60 / 83.33 & 88 / 78 & 12.12 / 89.40 & 7 / 69 & 82.19 / 87.20 & 37 / 40 \\
& & 10 & 96.94 / 90.26 & 97.35 / 89.37 & 99 / 95 & 41.26 / 90.53 & 23 / 60 & 13.21 / 88.97 & 10 / 76  & 65.35 / 84.36 & 76 / 74  & 11.42 / 91.44 & 6 / 77 & 74.17 / 90.74 & 32 / 48 \\\cline{2-16} 
& \multirow{3}{*}{DP-opt-out} & 1 & 99.66 / 83.40 & 13.29 / 82.41 & 4 / 100 & 99.98 / 77.35 & 98 / 42 & 12.81 / 64.46 & 2 / 66 & 95.95 / 71.76 & 73 / 97 & 99.98 / 79.22 & 97 / 42  & 48.14 / 10.68 & 37 / 37 \\
& & 5 & 99.95 / 90.86 & 10.66 / 91.19 & 4 / 100  & 99.99 / 91.04 & 98 / 42 & 11.36 / 89.50 & 2 / 70  & 99.80 / 88.11 & 89 / 95  & 100.0 / 92.12 & 98 / 42  & 91.89 / 88.19 & 45 / 44 \\
& & 10 & 99.98 / 92.42 & 21.31 / 90.06 & 22 / 96  & 100.0 / 92.75 & 97 / 42  & 11.21 / 91.20 & 4 / 78 & 99.69 / 88.28 & 82 / 89 & 99.99 / 92.85 & 99 / 42  & 88.22 / 90.41 & 41 / 40 \\\cline{2-16} 
& \multirow{3}{*}{Our Attack} & 1 & \cellcolor{gray!15}99.45 / \textbf{84.13} & \cellcolor{gray!15}\textbf{99.29} / \textbf{86.13} & 100 / 100 & 99.91 / 75.62 & 98 / 42 & \cellcolor{gray!15}\textbf{99.91} / \textbf{72.44} & 94 / 43 & \cellcolor{gray!15}\textbf{99.45} / \textbf{79.42} & 94 / 93 & 99.85 / 77.31 & 89 / 45 & \cellcolor{gray!15}\textbf{98.92} / \textbf{34.64} & 40 / 38 \\
& & 5 & \cellcolor{gray!15}\textbf{99.98} / \textbf{91.70} & \cellcolor{gray!15}\textbf{99.99} / \textbf{91.37} & 100 / 99 & \cellcolor{gray!15}\textbf{100.0} / \textbf{91.52} & 97 / 42 & \cellcolor{gray!15}\textbf{99.99} / \textbf{90.42} & 92 / 44  & 98.12 / 83.85 & 79 / 84 & \cellcolor{gray!15}\textbf{100.0} / 91.84 & 85 / 46 & \cellcolor{gray!15}\textbf{99.96} / \textbf{89.68} & 40 / 38\\
& & 10 & \cellcolor{gray!15}\textbf{99.99} / 91.69 & \cellcolor{gray!15}\textbf{99.97} / \textbf{90.88} & 99 / 97 & 14.43 / 90.21 & 8 / 64 & \cellcolor{gray!15}11.19 / \textbf{91.60} & 9 / 75  & \cellcolor{gray!15}99.40 / \textbf{89.29} & 86 / 85  & 10.88 / 90.04 & 9 / 79  & \cellcolor{gray!15}\textbf{99.52} / 88.18 & 39 / 38 \\
\midrule
\multirow{8}{*}{\shortstack{Sentiment\\-140}}
& \multirow{3}{*}{DP-opt-in} & 1 & 98.90 / 57.88 & 99.07 / 59.17 & 100 / 100  & 99.86 / 59.21 & 96 / 42 & 62.10 / 53.81 & 1 / 66 & 99.27 / 62.04 & 97 / 98 & 99.87 / 60.20 & 97 / 42  & 99.70 / 58.65 & 97 / 85 \\
& & 5 & 99.62 / 59.25 & 99.96 / 65.36 & 100 / 100 & 100.0 / 63.83 & 97 / 42  & 41.65 / 55.20 & 1 / 66 & 98.48 / 62.64 & 88 / 98  & 99.75 / 60.54 & 97 / 44 & 99.76 / 58.71 & 97 / 85 \\
& & 10 & 99.42 / 60.31 & 99.56 / 64.72 & 100 / 100 & 100.0 / 58.02 & 97 / 42  & 46.30 / 56.72 & 1 / 66 & 97.45 / 60.74 & 63 / 95 & 100.0 / 59.43 & 97 / 46 & 99.38 / 55.59 & 97 / 85 \\\cline{2-16} 
& \multirow{3}{*}{DP-opt-out} & 1 & 99.72 / 60.60 & 99.59 / 59.05 & 100 / 100 & 99.86 / 58.13 & 97 / 42 & 42.00 / 51.56 & 1 / 66  & 98.74 / 62.48 & 98 / 96 & 99.59 / 57.86 & 97 / 42 & 99.13 / 57.17 & 97 / 85 \\
& & 5 & 99.39 / 59.08 & 99.98 / 59.02 & 100 / 100 & 100.0 / 55.32 & 96 / 42  & 47.73 / 57.03 & 1 / 66 & 98.84 / 63.48 & 90 / 96  & 99.92 / 57.34 & 97 / 45  & 99.58 / 58.83 & 97 / 85 \\
& & 10 & 100.0 / 60.24 & 99.45 / 60.98 & 100 / 100 & 100.0 / 59.82 & 97 / 42 & 47.67 / 56.12 & 1 / 66 & 98.45 / 59.78 & 75 / 96 & 100.0 / 59.89 & 97 / 45 & 99.22 / 55.36 & 97 / 85 \\\cline{2-16} 
& \multirow{3}{*}{Our Attack} & 1 & \cellcolor{gray!15}\textbf{99.86} / \textbf{61.83} & 99.07 / 54.87 & 100 / 100  & 99.79 / 55.69 & 97 / 42 & \cellcolor{gray!15}\textbf{94.01} / 53.79 & 96 / 42 & 97.74 / 59.01 & 93 / 95  & 99.81 / 58.89 & 82 / 47  & 99.18 / 56.78 & 39 / 36 \\
& & 5 & \cellcolor{gray!15}99.61 / \textbf{62.54} & 99.61 / 61.65 & 100 / 100 & \cellcolor{gray!15}\textbf{100.0} / 55.01 & 96 / 42 & \cellcolor{gray!15}\textbf{97.55} / 54.77 & 96 / 42 & \cellcolor{gray!15}\textbf{99.08} / 58.58 & 83 / 95 & 99.50 / 56.77 & 84 / 50 & \cellcolor{gray!15}99.47 / \textbf{60.11} & 38 / 36 \\
& & 10 & \cellcolor{gray!15}99.76 / \textbf{60.80} & 99.27 / 60.82 & 100 / 100 & 99.87 / 58.79 & 97 / 42 & \cellcolor{gray!15}\textbf{99.86} / \textbf{59.05} & 95 / 42 & \cellcolor{gray!15}98.31 / \textbf{63.18} & 67 / 93  & \cellcolor{gray!15}99.38 / \textbf{63.10} & 84 / 52  & \cellcolor{gray!15}\textbf{100.0} / \textbf{63.05} & 38 / 36 \\
\bottomrule
\end{tabular}
}
\end{table*}

\subsection{Ablation Study}\label{sec:ablation} 
We conduct ablation studies on MNIST, CIFAR-10, and Sentiment-140 to evaluate \ring under seven dimensions of variation: the number of malicious clients $m$, the privacy budget $\epsilon$, the clipping bound $C$, the data distribution (iid vs. non-iid), the aggregation weights, the underlying backdoor technique, and client participation. The prob non-iid setting serves as the default unless otherwise specified.


\subsubsection{Impact of $m$}

Table~\ref{tab:ASR_ablation_m} reports ASR, Acc, and RR under varying numbers of malicious clients $m$. ASR increases with $m$ across all attacks, and several defenses show $m$-dependent retention behavior.

    

When there is no defense, ASR increases with $m$ for all attacks. \ring achieves ASR comparable to the DP-opt-out baseline, confirming that the coordinated noise cancellation in \ring can be successfully performed regardless of $m$. 

When the defense presents, \ring achieves ASR above $90\%$ in most configurations, with two exceptions at $m = 2$ under Flame and MESAS. When only two malicious clients collaborate, their noise vectors satisfy $\zeta_a = -\zeta_b$, which produces a geometrically distinctive pattern that both defenses exploit: the average RR of \ring drops to $44\%$ under Flame and $5\%$ under MESAS at $m = 2$. Increasing $m$ to 4 resolves this, improving ASR from $3\%$ to $93\%$ under MESAS and from $24\%$ to $94\%$ under Flame, consistent with the theoretical analayis in Section~\ref{sec:stealthy} that larger subgroups reduce inter-client noise correlation. Across other defense methods and values of $m$, \ring maintains consistently higher ASR than both baselines. Main-task accuracy is largely unaffected by changes in $m$, and the Acc of \ring remains on par with that of the baselines throughout.

\subsubsection{Impact of $\epsilon$}\label{sec:impact_epsilon}

Table~\ref{tab:ASR_ablation_e} reports ASR, Acc, and RR under varying privacy budgets $\epsilon$. The privacy budget primarily affects the DP-opt-in attack, while \ring remains robust across the range of $\epsilon$ examined.

The DP-opt-in attack is the most sensitive to $\epsilon$: without defenses, its ASR degrades substantially as $\epsilon$ decreases due to increased DP noise magnitude, and under defenses, excessively small $\epsilon$ suppresses the backdoor signal before it can benefit from the noise masking. The DP-opt-out attack maintains high ASR without defenses since no perturbation is applied to malicious updates, but is consistently identified by DeepSight, Flame, and FLShield regardless of $\epsilon$, as its unperturbed updates remain statistically anomalous. \ring avoids both failure modes: aggregation restores the backdoor signal independently of the per-client noise level, and the crafted updates remain statistically consistent with DP-perturbed benign ones across all values of $\epsilon$.

The one exception occurs at $\epsilon = 10$, where the low DP noise magnitude reduces the masking effect, causing a temporary ASR drop for \ring against Krum, Flame, and FreqFed, which is confirmed by the corresponding retention rate drop. ASR recovers as $\epsilon$ decreases. At $\epsilon=5$, \ring substantially outperforms DP-opt-in: $100\%$ vs. $77.49\%$ under Krum, $99.99\%$ vs. $69.25\%$ under Flame, and $100\%$ vs. $12.12\%$ under FreqFed.

As expected, Acc increases with $\epsilon$, and no significant difference is observed across the three attacks.

We additionally evaluate $\epsilon=0.5$ on MNIST and CIFAR-10. \ring retains high ASR in most settings; under Flame on CIFAR-10, ASR is low even though all malicious updates are retained. However, clean accuracy collapses, severely limiting model utility and the practical relevance of high ASR in this regime. Appendix~\ref{app:low_epsilon} provides the complete ASR and Acc results.

\begin{table*}[htbp]
\centering
\scriptsize
\vspace{-2pt}
\setlength{\tabcolsep}{3.15pt}
\caption{ASR, Acc (\%) and RR (\%) of \ring, DP-opt-in, and DP-opt-out attacks with varying $C$.}
\label{tab:ASR_ablation_c}
\resizebox{\textwidth}{!}{%
\begin{tabular}{c|ccccccccccccccc}
\toprule
\multirow{2}{*}{Dataset} &\multirow{2}{*}{}  & \multirow{2}{*}{$C$} & {No Defense}& \multicolumn{2}{c}{Deepsight} & \multicolumn{2}{c}{Krum} & \multicolumn{2}{c}{Flame} & \multicolumn{2}{c}{MESAS} & \multicolumn{2}{c}{FreqFed} & \multicolumn{2}{c}{FLShield} \\
& & & ASR / Acc & ASR / Acc &RR &ASR / Acc &RR &ASR / Acc &RR &ASR / Acc &RR &ASR / Acc &RR &ASR / Acc &RR\\
\midrule
\multirow9*{MNIST} & \multirow3*{DP-opt-in} & 1 & 72.31 / 61.61 & 79.88 / 44.80 & 100 / 75 & 31.81 / 59.23 & 27 / 59 & 42.80 / 58.94 & 43 / 70 & 91.79 / 45.01 & 90 / 81 & 32.38 / 59.97 & 31 / 77 & 36.60 / 62.99 & 22 / 53 \\
 & & 10 & 96.56 / 89.14 & 95.61 / 88.45 & 100 / 100 & 77.50 / 89.68 & 36 / 57 & 69.26 / 87.85 & 37 / 59 & 79.60 / 83.33 & 88 / 78 & 12.12 / 89.40 & 7 / 69 & 82.20 / 87.20 & 37 / 40 \\
 & & 20 & 97.77 / 92.10 & 98.25 / 91.74 & 100 / 100 & 85.28 / 89.07 & 39 / 56 & 90.03 / 88.85 & 38 / 57 & 84.83 / 83.48 & 89 / 78 & 86.04 / 89.10 & 36 / 58 & 91.97 / 85.13 & 39 / 38 \\
\cline{2-16}
 & \multirow3*{DP-opt-out} & 1 & 100.00 / 82.59 & 18.40 / 71.33 & 9 / 96 & 100.00 / 81.53 & 87 / 44 & 21.35 / 60.50 & 2 / 93 & 99.99 / 57.69 & 65 / 93 & 18.08 / 61.31 & 3 / 87 & 99.98 / 73.73 & 38 / 40 \\
 & & 10 & 99.95 / 90.87 & 10.66 / 91.19 & 4 / 100 & 99.99 / 91.05 & 98 / 42 & 11.36 / 89.51 & 2 / 70 & 99.80 / 88.12 & 89 / 95 & 100.00 / 92.12 & 98 / 42 & 91.90 / 88.20 & 45 / 44 \\
 & & 20 & 99.99 / 93.01 & 10.78 / 92.97 & 4 / 100 & 99.97 / 90.60 & 98 / 42 & 10.85 / 90.03 & 2 / 66 & 99.38 / 90.96 & 83 / 97 & 99.99 / 91.87 & 98 / 42 & 54.59 / 82.20 & 39 / 39 \\
\cline{2-16}
 & \multirow3*{Our Attack} & 1 & \cellcolor{gray!15}\textbf{100.00} / 81.74 & \cellcolor{gray!15}\textbf{86.46} / \textbf{71.61} & 42 / 93 & 19.06 / 60.26 & 1 / 66 & \cellcolor{gray!15}16.72 / \textbf{66.97} & 1 / 91 & \cellcolor{gray!15}\textbf{100.00} / \textbf{76.28} & 77 / 97 & 17.46 / 59.26 & 2 / 87 & \cellcolor{gray!15}\textbf{99.99} / \textbf{78.23} & 33 / 41 \\
 & & 10 & \cellcolor{gray!15}\textbf{99.98} / \textbf{91.71} & \cellcolor{gray!15}\textbf{100.00} / \textbf{91.37} & 100 / 99 & \cellcolor{gray!15}\textbf{100.00} / \textbf{91.53} & 97 / 42 & \cellcolor{gray!15}\textbf{100.00} / \textbf{90.43} & 92 / 44 & 98.12 / 83.86 & 79 / 84 & \cellcolor{gray!15}\textbf{100.00} / 91.85 & 85 / 46 & \cellcolor{gray!15}\textbf{99.96} / \textbf{89.69} & 40 / 38 \\
 & & 20 & \cellcolor{gray!15}\textbf{99.99} / 92.26 & \cellcolor{gray!15}\textbf{99.97} / 92.78 & 100 / 100 & \cellcolor{gray!15}\textbf{99.99} / \textbf{90.74} & 98 / 42 & \cellcolor{gray!15}\textbf{100.00} / \textbf{90.20} & 95 / 42 & 94.92 / 89.12 & 83 / 87 & \cellcolor{gray!15}\textbf{100.00} / 90.50 & 90 / 44 & \cellcolor{gray!15}\textbf{99.79} / \textbf{90.16} & 41 / 39 \\
\midrule
\multirow{8}{*}{\shortstack{Sentiment\\-140}} & \multirow3*{DP-opt-in} & 1 & 97.68 / 57.33 & 98.84 / 60.04 & 100 / 100 & 99.61 / 54.75 & 97 / 42 & 18.59 / 59.68 & 1 / 68 & 96.97 / 57.64 & 58 / 96 & 98.46 / 57.88 & 98 / 46 & 98.76 / 58.99 & 97 / 86 \\
 & & 10 & 99.63 / 59.25 & 99.97 / 65.36 & 100 / 100 & 100.00 / 63.83 & 97 / 42 & 41.65 / 55.21 & 1 / 66 & 98.49 / 62.65 & 88 / 98 & 99.75 / 60.55 & 97 / 44 & 99.77 / 58.72 & 97 / 85 \\
 & & 20 & 98.83 / 57.27 & 99.06 / 59.41 & 100 / 100 & 99.97 / 56.75 & 96 / 42 & 46.95 / 54.07 & 1 / 66 & 98.87 / 61.15 & 98 / 96 & 100.00 / 55.85 & 97 / 43 & 99.61 / 61.34 & 97 / 85 \\
\cline{2-16}
 & \multirow3*{DP-opt-out} & 1 & 98.92 / 61.92 & 99.23 / 59.59 & 100 / 100 & 100.00 / 61.09 & 97 / 42 & 70.50 / 60.13 & 1 / 68 & 98.07 / 54.16 & 33 / 97 & 99.61 / 53.41 & 98 / 45 & 98.90 / 57.35 & 97 / 85 \\
 & & 10 & 99.40 / 59.08 & 99.98 / 59.03 & 100 / 100 & 100.00 / 55.33 & 96 / 42 & 47.74 / 57.04 & 1 / 66 & 98.84 / 63.49 & 90 / 96 & 99.92 / 57.34 & 97 / 44 & 99.58 / 58.84 & 97 / 85 \\
 & & 20 & 99.38 / 61.51 & 99.18 / 62.69 & 100 / 100 & 100.00 / 56.58 & 96 / 42 & 64.65 / 55.23 & 1 / 66 & 98.59 / 62.31 & 98 / 96 & 99.97 / 61.92 & 97 / 43 & 99.40 / 58.88 & 97 / 85 \\
\cline{2-16}
 & \multirow3*{Our Attack} & 1 & \cellcolor{gray!15}\textbf{99.23} / 60.51 & \cellcolor{gray!15}\textbf{99.23} / 58.12 & 100 / 100 & \cellcolor{gray!15}\textbf{100.00} / \textbf{61.49} & 97 / 42 & \cellcolor{gray!15}\textbf{100.00} / 57.09 & 83 / 46 & \cellcolor{gray!15}\textbf{98.46} / 56.70 & 39 / 97 & \cellcolor{gray!15}\textbf{100.00} / \textbf{61.26} & 85 / 54 & \cellcolor{gray!15}\textbf{99.26} / 55.56 & 39 / 34 \\
 & & 10 & \cellcolor{gray!15}99.61 / \textbf{62.55} & 99.61 / 61.65 & 100 / 100 & \cellcolor{gray!15}\textbf{100.00} / 55.02 & 96 / 42 & \cellcolor{gray!15}\textbf{97.56} / 54.77 & 96 / 42 & \cellcolor{gray!15}\textbf{99.09} / 58.58 & 83 / 95 & 99.51 / 56.77 & 84 / 50 & \cellcolor{gray!15}99.47 / \textbf{60.11} & 38 / 36 \\
 & & 20 & \cellcolor{gray!15}\textbf{99.44} / \textbf{64.91} & \cellcolor{gray!15}98.53 / \textbf{63.97} & 100 / 100 & \cellcolor{gray!15}99.94 / \textbf{59.92} & 96 / 42 & \cellcolor{gray!15}\textbf{96.59} / \textbf{56.11} & 96 / 42 & \cellcolor{gray!15}\textbf{99.80} / 62.26 & 94 / 96 & 99.64 / 58.85 & 82 / 48 & 99.41 / 61.14 & 39 / 36 \\

\bottomrule
\end{tabular}
}
\vspace{-3pt}
\end{table*}

\begin{table*}[htbp]
\centering
\scriptsize
\vspace{-2pt}
\setlength{\tabcolsep}{3.15pt}
\caption{ASR, Acc (\%) and RR(\%) of \ring, DP-opt-in, and DP-opt-out attacks under different data distributions.}
\label{tab:ASR_ablation_dist}
\resizebox{\textwidth}{!}{%
\begin{tabular}{c|ccccccccccccccc}
\toprule
\multirow{2}{*}{Dataset} & \multirow{2}{*}{} & \multirow{2}{*}{Dist.} & {No Defense} & \multicolumn{2}{c}{DeepSight} & \multicolumn{2}{c}{Krum} & \multicolumn{2}{c}{FLAME} & \multicolumn{2}{c}{MESAS} & \multicolumn{2}{c}{FreqFed} & \multicolumn{2}{c}{FLShield} \\
& & & ASR / Acc & ASR / Acc & RR & ASR / Acc & RR & ASR / Acc & RR & ASR / Acc & RR & ASR / Acc & RR & ASR / Acc & RR \\
\midrule
\multirow{6}{*}{MNIST}
& \multirow{2}{*}{DP-opt-in}
& iid  & 94.73 / 91.07 & 90.74 / 91.14 & 100 / 100 & 65.88 / 91.24 & 31 / 59 & 11.04 / 92.10 & 2 / 67 & 82.32 / 89.89 & 86 / 89 & 9.43 / 91.47 & 2 / 67 & 87.66 / 91.11 & 35 / 43 \\
& & prob & 96.56 / 89.14 & 95.61 / 88.45 & 100 / 100 & 77.50 / 89.68 & 36 / 57 & 69.26 / 87.85 & 37 / 59 & 79.60 / 83.33 & 88 / 78 & 12.12 / 89.40 & 7 / 69 & 82.20 / 87.20 & 37 / 40 \\
\cline{2-16}
& \multirow{2}{*}{DP-opt-out}
& iid  & 99.98 / 92.86 & 7.37 / 91.06 & 4 / 100 & 100.00 / 93.44 & 98 / 42 & 11.14 / 90.86 & 2 / 67 & 99.65 / 91.11 & 81 / 97 & 100.00 / 93.34 & 98 / 42 & 93.45 / 90.73 & 42 / 42 \\
& & prob & 99.95 / 90.87 & 10.66 / 91.19 & 4 / 100 & 99.99 / 91.05 & 98 / 42 & 11.36 / 89.51 & 2 / 70 & 99.80 / 88.12 & 89 / 95 & 100.00 / 92.12 & 98 / 42 & 91.90 / 88.20 & 45 / 44 \\
\cline{2-16}
& \multirow{2}{*}{Our Attack}
& iid  & \cellcolor{gray!15}\textbf{100.00} / 92.74 & \cellcolor{gray!15}\textbf{99.94} / \textbf{93.10} & 100 / 100 & \cellcolor{gray!15}\textbf{100.00} / 93.09 & 98 / 42 & \cellcolor{gray!15}\textbf{100.00} / \textbf{92.29} & 96 / 43 & \cellcolor{gray!15}\textbf{99.90} / \textbf{91.67} & 88 / 93 & 99.99 / 92.46 & 74 / 50 & \cellcolor{gray!15}\textbf{99.99} / \textbf{92.65} & 40 / 39 \\
& & prob & \cellcolor{gray!15}\textbf{99.98} / \textbf{91.71} & \cellcolor{gray!15}\textbf{100.00} / \textbf{91.37} & 100 / 99 & \cellcolor{gray!15}\textbf{100.00} / \textbf{91.53} & 97 / 42 & \cellcolor{gray!15}\textbf{100.00} / \textbf{90.43} & 92 / 44 & 98.12 / 83.86 & 79 / 84 & \cellcolor{gray!15}\textbf{100.00} / 91.85 & 85 / 46 & \cellcolor{gray!15}\textbf{99.96} / \textbf{89.69} & 40 / 38 \\
\midrule
\multirow{6}{*}{\shortstack{Sentiment\\-140}}
& \multirow{2}{*}{DP-opt-in}
& iid  & 99.54 / 61.17 & 99.00 / 59.94 & 100 / 100 & 99.92 / 61.45 & 98 / 42 & 53.98 / 56.59 & 2 / 67 & 98.53 / 61.68 & 90 / 98 & 99.69 / 53.80 & 97 / 44 & 99.77 / 63.69 & 98 / 86 \\
& & prob & 99.63 / 59.25 & 99.97 / 65.36 & 100 / 100 & 100.00 / 63.83 & 97 / 42 & 41.65 / 55.21 & 1 / 66 & 98.49 / 62.65 & 88 / 98 & 99.75 / 60.55 & 97 / 44 & 99.77 / 58.72 & 97 / 85 \\
\cline{2-16}
& \multirow{2}{*}{DP-opt-out}
& iid  & 98.61 / 59.61 & 99.69 / 57.71 & 100 / 100 & 99.69 / 57.21 & 98 / 42 & 48.03 / 56.98 & 2 / 67 & 99.07 / 61.45 & 91 / 96 & 99.92 / 59.27 & 98 / 44 & 99.23 / 63.52 & 97 / 85 \\
& & prob & 99.40 / 59.08 & 99.98 / 59.03 & 100 / 100 & 100.00 / 55.33 & 96 / 42 & 47.74 / 57.04 & 1 / 66 & 98.84 / 63.49 & 90 / 96 & 99.92 / 57.34 & 97 / 44 & 99.58 / 58.84 & 97 / 85 \\
\cline{2-16}
& \multirow{2}{*}{Our Attack}
& iid  & 99.38 / \cellcolor{gray!15}\textbf{64.19} & 99.61 / \cellcolor{gray!15}\textbf{60.45} & 100 / 100 & \cellcolor{gray!15}\textbf{100.00} / 60.17 & 98 / 42 & \cellcolor{gray!15}\textbf{97.61} / 56.09 & 97 / 42 & \cellcolor{gray!15}\textbf{99.31} / \textbf{62.35} & 86 / 96 & 99.85 / \cellcolor{gray!15}\textbf{59.50} & 84 / 50 & 99.69 / 60.39 & 31 / 35 \\
& & prob & 99.61 / \cellcolor{gray!15}\textbf{62.55} & 99.61 / 61.65 & 100 / 100 & \cellcolor{gray!15}\textbf{100.00} / 55.02 & 96 / 42 & \cellcolor{gray!15}\textbf{97.56} / 54.77 & 96 / 42 & \cellcolor{gray!15}\textbf{99.09} / 58.58 & 83 / 95 & 99.51 / 56.77 & 84 / 50 & 99.47 / \cellcolor{gray!15}\textbf{60.11} & 38 / 36 \\
\bottomrule
\end{tabular}%
}
\vspace{-3pt}
\end{table*}

\subsubsection{Impact of $C$} 

Table~\ref{tab:ASR_ablation_c} reports attack performance under varying clipping bounds $C$ in DP-SGD. The impact of $C$ on the three attacks closely parallels the pattern observed for $\epsilon$.

The DP-opt-in attack is the most sensitive to $C$. Tighter clipping reduces gradient magnitude, which weakens the injected backdoor signal both directly and by requiring lower per-round noise magnitude. Without defenses its ASR drops substantially at small $C$, and under defenses it remains fragile even though DP perturbation retains more malicious updates than the DP-opt-out baseline. The DP-opt-out attack maintains high ASR without defenses, but is consistently detected by DeepSight and Flame regardless of $C$, since its unperturbed updates remain statistically anomalous. \ring avoids both failure modes. The backdoor signal is restored during aggregation regardless of per-client clipping, and the crafted updates remain statistically consistent with DP-perturbed benign ones. 

The exception occurs at $C = 1$, where low per-round noise magnitude diminishes the masking effect, leading to reduced ASR for \ring against Krum, Flame, and FreqFed, consistent with the corresponding retention rate drops. This behavior mirrors the $\epsilon = 10$ case. As $C$ increases, the masking effect quickly recovers. At $C = 10$, \ring substantially outperforms DP-opt-in: $100\%$ vs. $77.50\%$ under Krum, $100\%$ vs. $69.26\%$ under Flame, and $100\%$ vs. $12.12\%$ under FreqFed.

As expected, Acc increases with $C$ since stronger clipping imposes less distortion on benign gradients. Main-task accuracy under \ring remains consistent with that of the baselines across all tested settings.

\subsubsection{iid vs. non-iid}


Table~\ref{tab:ASR_ablation_dist} reports ASR, Acc, and RR under iid and non-iid data distributions. \ring remains effective across both settings, confirming that its stealthiness mechanism is robust to data heterogeneity.


Without defenses, all three attacks achieve high ASR under both iid and non-iid configurations. \ring and DP-opt-out maintain stronger and more consistent effectiveness than DP-opt-in, with \ring achieving nearly $100\%$ ASR under both iid and non-iid settings.

Under defenses, non-iid data can make detection more challenging, as benign updates naturally become more diverse and harder to separate from malicious ones. \ring exploits this effect and maintains high ASR across most defenses and distributions. On MNIST, \ring achieves at least $98\%$ ASR under all evaluated defenses in both iid and non-iid settings. DP-opt-out, by contrast, remains vulnerable to DeepSight and Flame despite its strong no-defense performance, and DP-opt-in, while benefiting from DP noise masking in some configurations, produces a weaker and less reliable attack signal, particularly against Flame and FreqFed. On Sentiment-140, \ring delivers strong and stable performance across most defenses, with a particularly notable advantage over both baselines under Flame. Taken together, these results confirm that \ring preserves the backdoor signal while maintaining stealthiness across diverse data distributions and defense configurations.

Main-task accuracy remains comparable across iid and non-iid settings. Where non-iid data reduces Acc due to client heterogeneity, the utility under \ring stays on par with that of the baselines.



\subsubsection{Unequal Aggregation Weights}
\label{sec:unequal_weights}
{Our main experiments use data-size-weighted FedAvg, with both benign and malicious clients' base weights proportional to their sample counts under the non-iid partition. To test greater weight imbalance, we perform a controlled sensitivity analysis using randomized client weights with a tunable dispersion parameter; without additional defenses, \ring's ASR remains stable across the tested range (Appendix~\ref{app:unequal_weights}).

\subsubsection{Generality across Backdoor Techniques}
To validate that \ring is agnostic to the choice of backdoor technique, we evaluate it with three representative image backdoor attacks on MNIST: visible-trigger backdoor (VBA)~\cite{gu2017badnets}, distributed backdoor attack (DBA)~\cite{wang2020attack}, and Neurotoxin (NBA)~\cite{pmlr-v162-zhang22w}. As shown in Table~\ref{tab:ASR_vary_backdoor} (Appendix~\ref{app_sec:results_backdoor}), \ring consistently improves attack effectiveness over the DP-opt-in baseline and matches or exceeds the DP-opt-out baseline under most defenses, with no degradation in clean accuracy. These results are consistent with the analysis in Section~\ref{sec:stealthy}: the underlying backdoor technique governs the task-gradient component, while \ring operates at the perturbation level to improve stealthiness and enable aggregation-level backdoor recovery. 

\subsubsection{Random Client Participation}
We further evaluate \ring under random client participation on MNIST and CIFAR-10. ASR exceeds 97\% with $\beta=20\%$ across all evaluated settings, consistent with the main results; at $\beta=10\%$, it ranges from 80.96\% to 98.30\%. At 5\%, less frequent malicious co-participation is accompanied by lower ASR and greater variation across defenses. Detailed settings and ASR/Acc results are provided in Appendix~\ref{app:random_participation}.

\subsection{Interplay between Privacy, Security and Utility}
Our ablation study reveals how key parameters jointly shape the relationship between privacy, security, and utility in DP-FL. 

The number of malicious clients $m$ is attacker-controlled and primarily determines attack strength: as $m$ grows, the system becomes less robust against \ring.  The privacy and utility tradeoff still holds here as a lower $\epsilon$ strengthens privacy but also reduces the main task accuracy. Unlike the DP-opt-in where a small  $\epsilon$ enhances FL security, our \ring attack is immune to the DP perturbation and remains effective regardless of the value of $\epsilon$. The clipping bound $C$ affects both security and utility but not the privacy directly. Smaller $C$ reduces main-task utility and lowers gradient sensitivity, which in turn reduces the per-round noise magnitude and slightly weakens the masking effect, leading to a modest improvement in system robustness against \ring at the cost of utility.

\section{Mitigation}\label{sec:mitigation}
In this section, we discuss potential mitigation strategies and highlight opportunities and challenges.

\mybpara{Detecting Malicious Users.} 
Figure~\ref{fig:RR_e=5_m=6} shows that Krum, Flame, and FreqFed may inadvertently keep a higher proportion of malicious updates in the prob non-iid setting, suggesting that these methods might be capable of differentiating malicious and benign updates, but in the wrong direction. However, aligning with our theoretical analysis, this distinction fades away once the underlying data distribution changes, as shown in Figure~\ref{apx_fig:RR_e=5_m=6_dir} in Appendix~\ref{app_sec:results_backdoor}. In addition, our analysis in Section~\ref{sec:stealthy} indicates that \ring attack reduces per-client noise variance by a factor of $(m_l-1)/m_l$, creating only a subtle variance gap compared to benign updates.  While this gap becomes more detectable with a larger $\epsilon$, increasing $\epsilon$ simultaneously weakens privacy guarantees, which is concerning if DP is primarily used for privacy enhancement.

\begin{figure}[htbp]
\centering
\vspace{-2pt}
\includegraphics[scale=0.4]{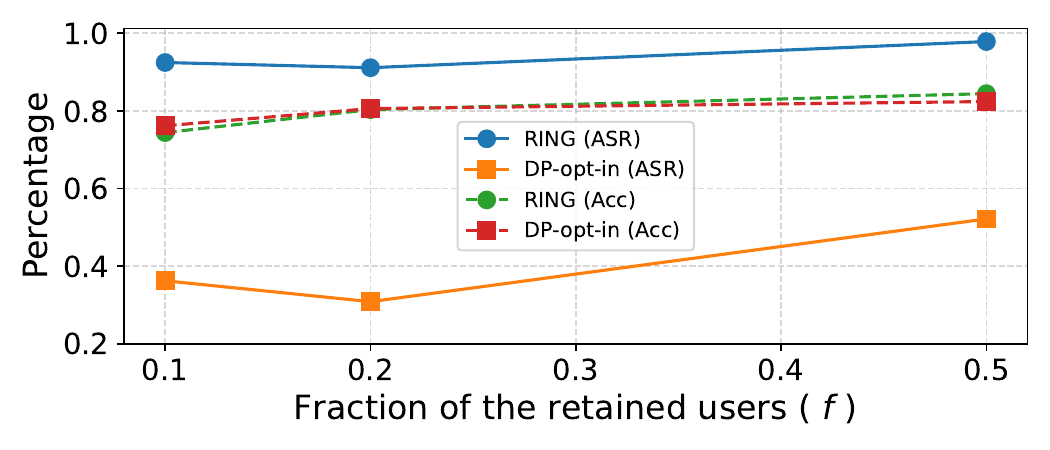}
\caption{Impact of $f$ on ASR and Acc of \ring and DP-opt-in attack under the prob non-iid setting on MNIST ($\epsilon=1, m=4, \eta=0.01$), without defense.}
\label{fig:rd_e=1_m=6}
\vspace{-10pt}
\end{figure}

\mybpara{Random Removal of Local Updates.} 
An intuitive mitigation strategy is to randomly retain only a fraction of received benign and malicious updates, to disrupt noise cancellation. Section~\ref{sec:effectiveness} provides a theoretical characterization of this effect: for the same retention probability $f$, \ring produces a substantially smaller residual noise magnitude in the aggregated gradient than the DP-opt-in attack, and therefore achieves higher attack gain even under random removal.

Figure~\ref{fig:rd_e=1_m=6} confirms this empirically. \ring consistently outperforms DP-opt-in across all tested values of $f$, consistent with Theorem~\ref{thm:err_compare}. This advantage stems from partial noise cancellation: even when some malicious updates are dropped, the remaining perturbations cancel sufficiently to recover a strong backdoor signal, which is the core effectiveness objective of \ring. As $f$ decreases, the ASR of \ring degrades only slightly, whereas main-task accuracy drops more noticeably, suggesting that random removal imposes a greater utility cost than security benefit against \ring.
\mybpara{Limiting Response Time.} 
A potential server-side defense is to enforce a strict deadline for submitting local updates. Since \ring requires peer coordination among malicious clients, their submission latency is higher than that of benign clients, and a sufficiently tight cutoff time could exclude malicious updates from aggregation. This timing-based approach also applies to other coordinated attacks such as DBA~\cite{xie2019dba}. For example, with CIFAR-10/ResNet-18 ($\approx 44.8$~MB) over a $1$~Gbps link, one communication round for a benign client takes approximately $0.72s$, while a malicious client incurs roughly $1.44s$ due to peer interaction overhead. Local training adds approximately $5.2s$, with an additional $20ms$ for peer coordination in our experiment. These illustrative estimates suggest a threshold of approximately $6s$ could separate benign from malicious submissions. In practice, however, this approach requires a perfectly synchronized environment with homogeneous compute and network conditions, a prerequisite rarely satisfied in real FL deployments. An overly aggressive threshold may introduce systematic bias toward resource-rich clients, degrading global model performance.

\begin{figure}[htbp]
\centering
\vspace{-4pt}
\includegraphics[scale=0.4]{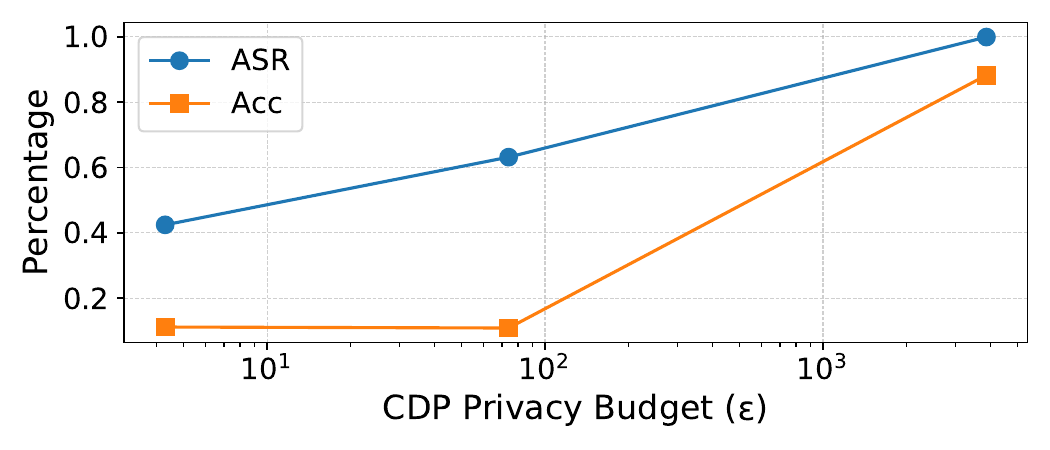}
\caption{Impact of server-side DP on ASR and Acc under the prob non-iid setting on MNIST.}
\label{fig:cdp_e=5_m=6}
\vspace{-10pt}
\end{figure}

\mybpara{Server-side Perturbation.}  
Another potential defense is to introduce server-side noise to disrupt the noise cancellation property of \ring. Existing methods such as Flame already apply a small amount of post-aggregation noise, but \ring remains effective in this setting, as reflected in the retention rates of Figure~\ref{fig:RR_e=5_m=6}. Existing attacks also expose limitations of server-side noise-based defenses: edge-case backdoors can be suppressed by weak DP at the cost of clean accuracy~\cite{wang2020attack}, while the attack remains effective against FLAME~\cite{wang2024flare}. A stronger variant is to apply central DP, for example, server-side DP-SGD, after aggregation to inject calibrated noise. As shown in Figure~\ref{fig:cdp_e=5_m=6}, tightening the central privacy budget from $\epsilon \approx 3{,}872$ (near-negligible privacy) to $\epsilon \approx 4.3$ (strong privacy) reduces ASR from $99.91\%$ to $42.45\%$, but also drops Acc from $88.24\%$ to $11.18\%$. This severe utility cost arises because benign updates incur greater total noise distortion than malicious ones: benign clients are subject to both client-side and server-side perturbation, whereas malicious clients are only affected by the server-side DP noise. Beyond utility, server-side perturbation complicates the analysis of the system-wide privacy budget. Furthermore, central DP cannot protect against a curious server, and is therefore misaligned with our threat model.

\section{Discussion}
\mybpara{Privacy Granularity.}
Sample-level and client-level DP specify the protection unit, whereas central and local DP distinguish trust models. We focus on sample-level DP implemented locally via DP-SGD. The evaluated defenses inspect individual client updates, whereas the DP guarantee protects individual training samples. Under central client-level DP without local perturbation, there is no client-side DP noise for \ring to mimic. Local client-level DP retains a potential noise cloak but requires client-level sensitivity bounds and noise calibration; evaluating \ring in this setting remains future work.


\mybpara{Secure Aggregation.} Secure aggregation (SA)~\cite{segal2017practical} enhances privacy by leveraging cryptographic protocols such as multi-party computation, preventing the server from observing individual client updates and exposing only the gradient aggregate. SA can be applied independently or combined with DP for stronger privacy guarantees~\cite{chen2022poisson,chen2022fundamental,stevens2022efficient}. Crucially, SA does not alter the aggregation result itself. When used alone, the backdoor signal remains intact; when combined with DP, attack performance is expected to match that in standard differentially private FL. The defenses evaluated in our experiments and most backdoor defenses in the literature operate at the client level, inspecting individual updates prior to aggregation. SA renders this class of defenses ineffective, as the cryptographic obfuscation of individual updates prevents the server from accessing the information these methods rely on.

\mybpara{Weaponizing DP in Prior Research.}
Prior work shows that LDP is vulnerable to data poisoning attacks, where adversaries directly manipulate the perturbation function to control the aggregated result~\cite{cao2021data, li2023fine, li2024robustness, Li2025MDPA}. Similarly, \cite{giraldo2020adversarial} study adversarial classification under DP, showing that DP noise can serve as cover for false-data injection: an attacker maximizes numerical bias while keeping the attack distribution close to benign DP outputs under a known detector. These attacks are, however, constrained by the DP noise scale and stealth budget — analogous to DP-opt-in backdoor attacks in FL, where DP noise improves local stealth at the cost of weakening the malicious signal. 

Our proposed \ring attack takes a fundamentally different approach. Rather than targeting LDP or low-dimensional analytical outputs, it attacks sample-level DP-SGD in FL and operates on high-dimensional model updates without any knowledge of the deployed defense. Malicious clients coordinate perturbations with Gaussian marginals resembling DP noise, bypassing existing defenses. Critically, these perturbations cancel during aggregation, recovering a backdoor effect close to the noise-free DP-opt-out case. \ring thus simultaneously achieves DP-noise-like local stealth and a strong aggregate backdoor signal — a threat that grows more significant as DP becomes a standard component of privacy-preserving FL systems.
\section{Conclusion}
This work challenges the effectiveness of differential privacy as a defense against backdoor attacks in federated learning. We introduce \ring, a novel backdoor attack targeting differentially private FL, and demonstrate through theoretical analysis and empirical evaluation that it evades state-of-the-art defenses while substantially increasing attack success rate. Our findings unfold a new threat landscape that DP can be exploited to weaken, rather than strengthen, the security of AI/ML systems. The result also reveals a fundamental tension between security, privacy, and utility. Service providers who reduce or abandon DP in response to such threats sacrifice the privacy guarantees it was designed to provide -- a trade-off that underscores the urgent need for defenses that address security and privacy in concert.

\section{Acknowledgments}
This work was supported in part by NSF grants CNS-2238680, CNS-2207204, and CNS‑2247794. We acknowledge the
computing resources and support provided by Applied AI Research
Center (AARC) and the Research Center for Integrated Cybersecurity (InCy) at Purdue University. The work was also supported by allocation CIS250794 from the Advanced Cyberinfrastructure
Coordination Ecosystem: Services \& Support (ACCESS) program, which is supported
by NSF grants \#2138259, \#2138286, \#2138307,
\#2137603, and \#2138296~\cite{boerner2023access}. 
\bibliographystyle{IEEEtran}
\bibliography{cite}

\begin{appendices}


\begin{figure}[htbp]
\centering
\scalebox{0.9}[0.85]{
\includegraphics[scale=0.2]{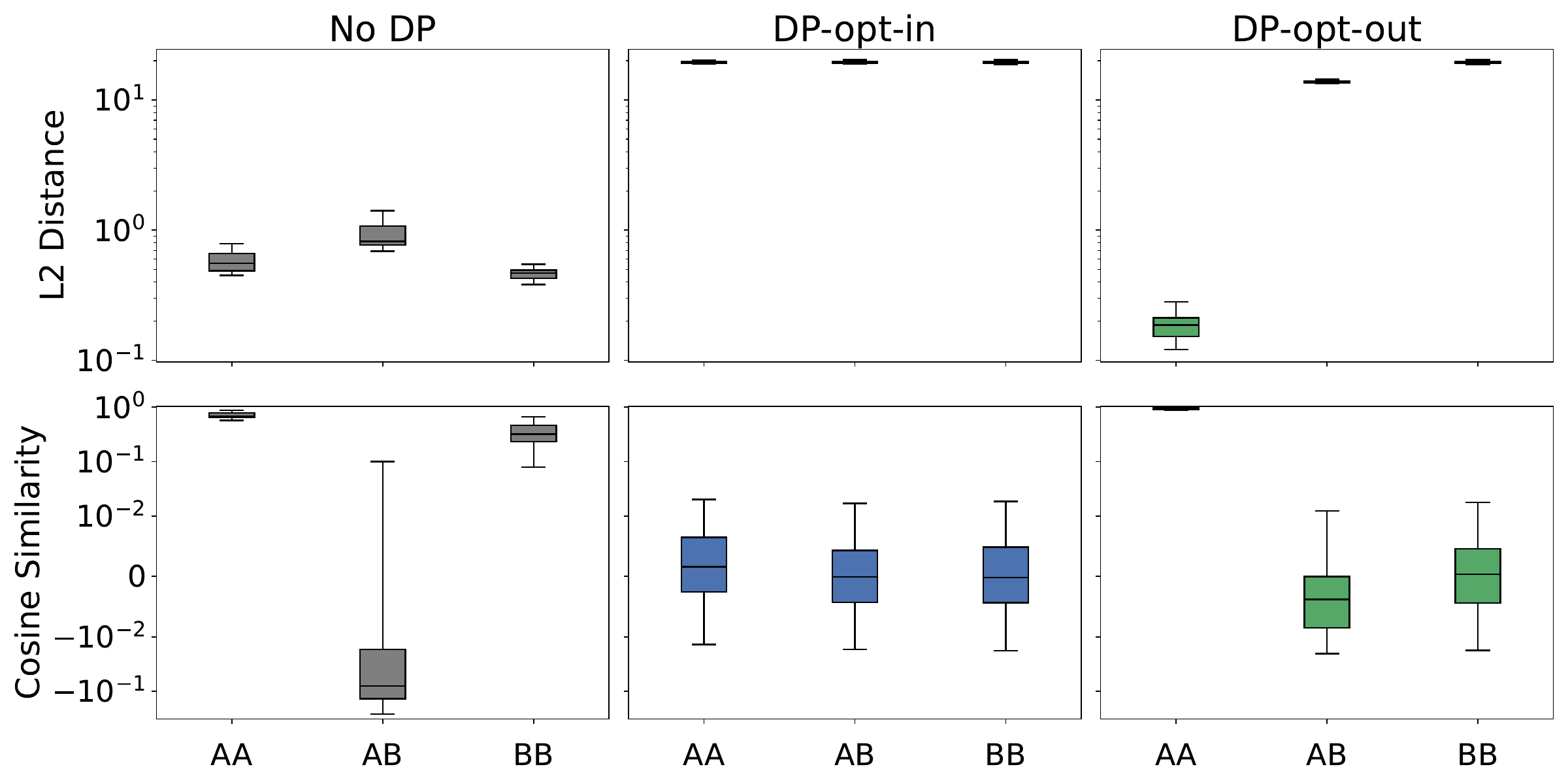}
}
\vspace{-5pt}
\caption{Boxplots of pairwise $L_2$ distance and cosine similarity for client update pairs under No DP, DP-opt-in, and DP-opt-out.}
 \label{fig_geometry_input&output}
\end{figure}
\vspace{-15pt}
\section{Impact of DP Noise on Client Update Geometry}\label{app:Geometry}
To explain why DP-opt-in undermines detection-based defenses, we analyze the pairwise geometry of client updates in Figure~\ref{fig_geometry_input&output}. This analysis is directly motivated by how existing defenses operate: Krum relies on pairwise Euclidean distances, Flame and FreqFed use cosine-based clustering, and MESAS combines both. Even for DeepSight and FLShield, which inspect model behavior rather than raw parameter geometry, anomalous update geometry serves as an important discriminative signal. Accordingly, $\ell_2$ distance and cosine similarity together provide natural and representative geometric lenses for understanding defense evasion.

In the no-DP setting, the three pairwise update types exhibit clearly separable geometric patterns. Attacker-attacker (AA) pairs cluster tightly with high cosine similarity and small $\ell_2$ distance. Attacker-benign (AB) pairs show lower similarity and larger distance, while benign-benign (BB) pairs remain moderately well clustered. These geometric separations are the statistical cues that detection-based defenses exploit.

Under DP-opt-in, however, injected DP noise collapses these distinctions. The cosine similarity distributions of AA, AB, and BB pairs all concentrate near zero, and their $L_2$ distances become uniformly large and mutually overlapping. These geometric metrics show substantial overlap between malicious and benign updates. In contrast, the DP-opt-out attacker preserves the tight alignment of AA pairs, maintaining a detectable geometric anomaly that defenses can, in principle, exploit, but at the cost of the stealthiness that DP noise would otherwise provide.
\vspace{-5pt}
\section{Proof of Theorem \ref{theorem:noise_component}}
\label{app:Proof_1}
\vspace{-5pt}
\begin{proof}
    For each $j\in\mathcal{M}_t$,
    \begin{align} \nonumber
        \zeta_j &= \left[z_j - \frac{1}{m_l}\sum_{k=1}^{m_l} z_k\right] \\ \nonumber
        &= \left[z_j - \frac{1}{m_l} z_j - \frac{1}{m_l}\sum_{k\neq j}z_k\right] \\ \nonumber
        &= \left[(1-\frac{1}{m_l})\cdot z_j - \frac{1}{m_l}\sum_{k\neq j}z_k\right] \nonumber
    \end{align}
    Since all $z_k$ are independent with variance $\sigma^2$, the variance of $\zeta_j$ is 
    \begin{align}\nonumber
        Var[\zeta_j] &= \left[(1-\frac{1}{m_l})^2 \sigma^2 + (m_l-1)(\frac{1}{m_l})^2\sigma^2\right]\\ 
        &=\frac{m_l-1}{m_l}\cdot \sigma^2.\nonumber
    \end{align}
    Therefore, $\zeta_j\sim \mathcal{N}(0,\frac{m_l-1}{m_l} \sigma^2I_d)$.
\end{proof}

\begin{figure*}[htbp]
\centering
\scalebox{0.9}[0.65]{
\includegraphics[scale=0.35]{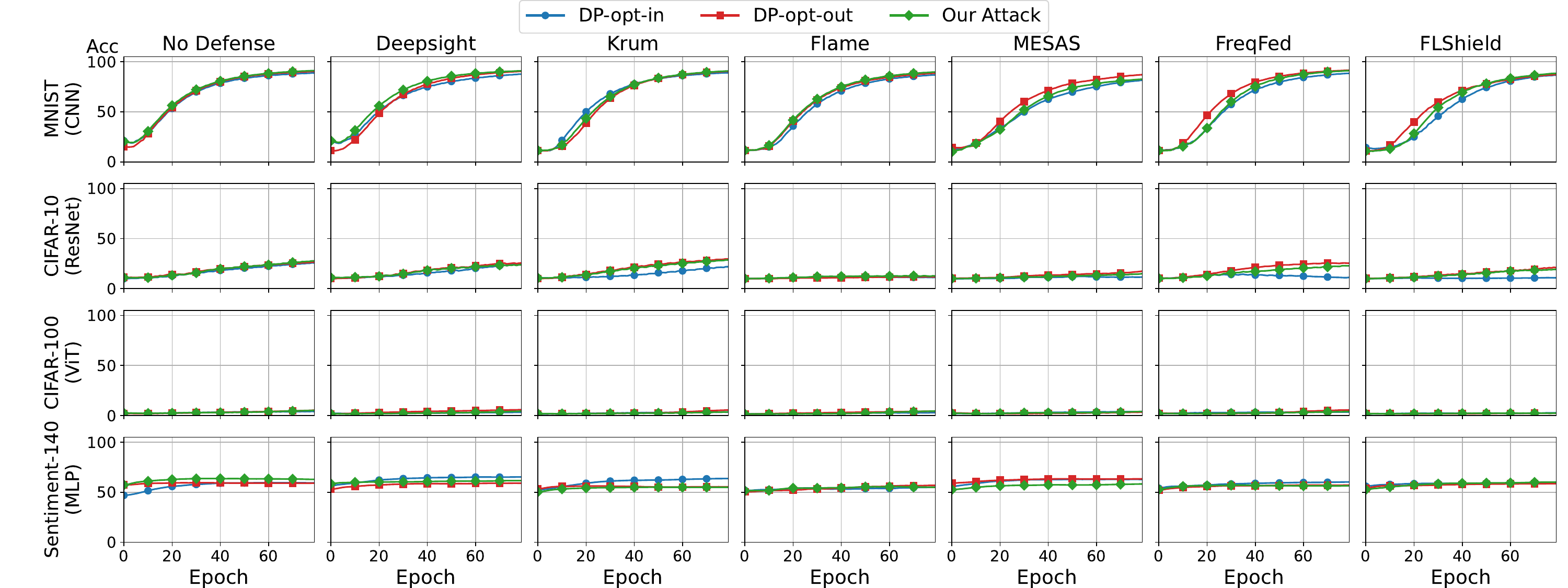}
}
 \caption{Acc of DP-opt-in, DP-opt-out and \ring attacks in the prob non-iid setting.}
 \label{fig:Acc_e=5_m=6}
\end{figure*}

\begin{figure*}[htbp]
\centering
\scalebox{0.9}[0.65]{
\includegraphics[scale=0.35]{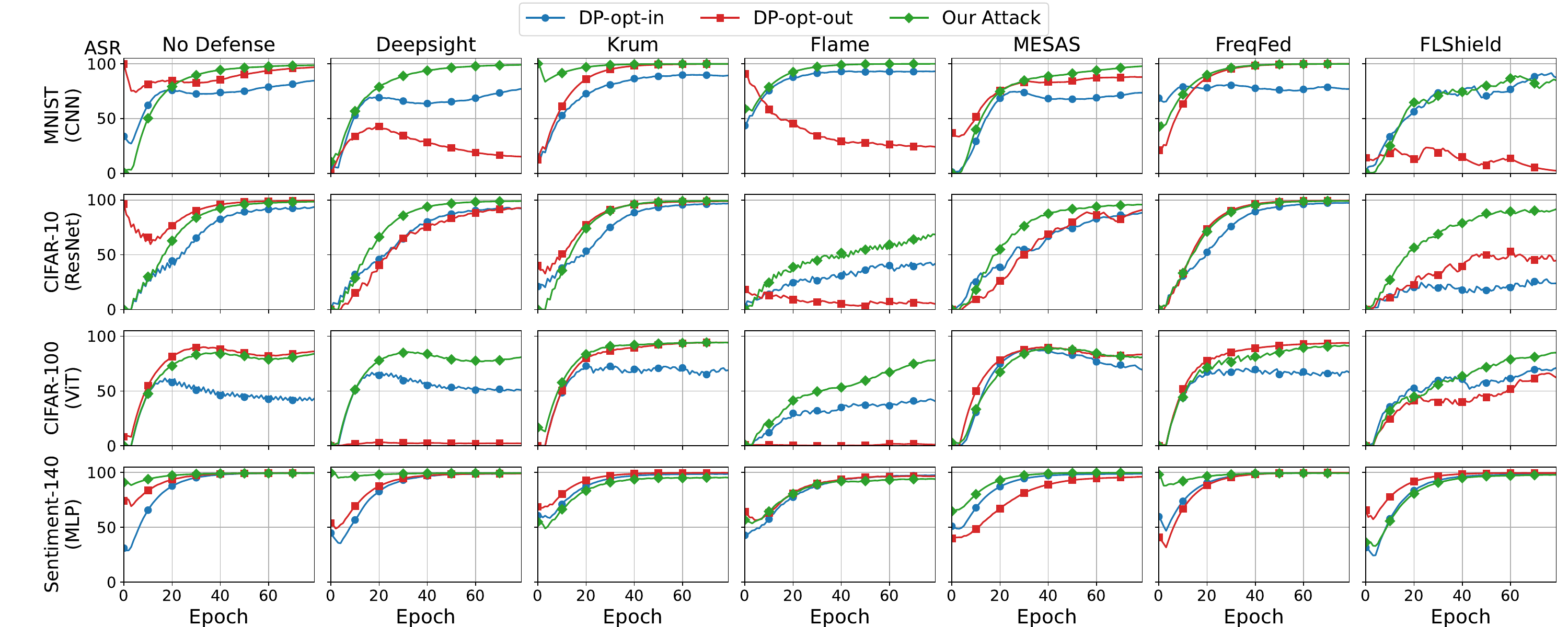}
}
 \caption{ASR of DP-opt-in, DP-opt-out, and \ring attacks in the qty non-iid setting.}
 \label{apx_fig:ASR_e=5_m=6_qty}
\end{figure*}

\begin{figure*}[htbp]
\centering
 \vspace{-5pt}
 \scalebox{0.9}[0.65]{
\includegraphics[scale=0.304]{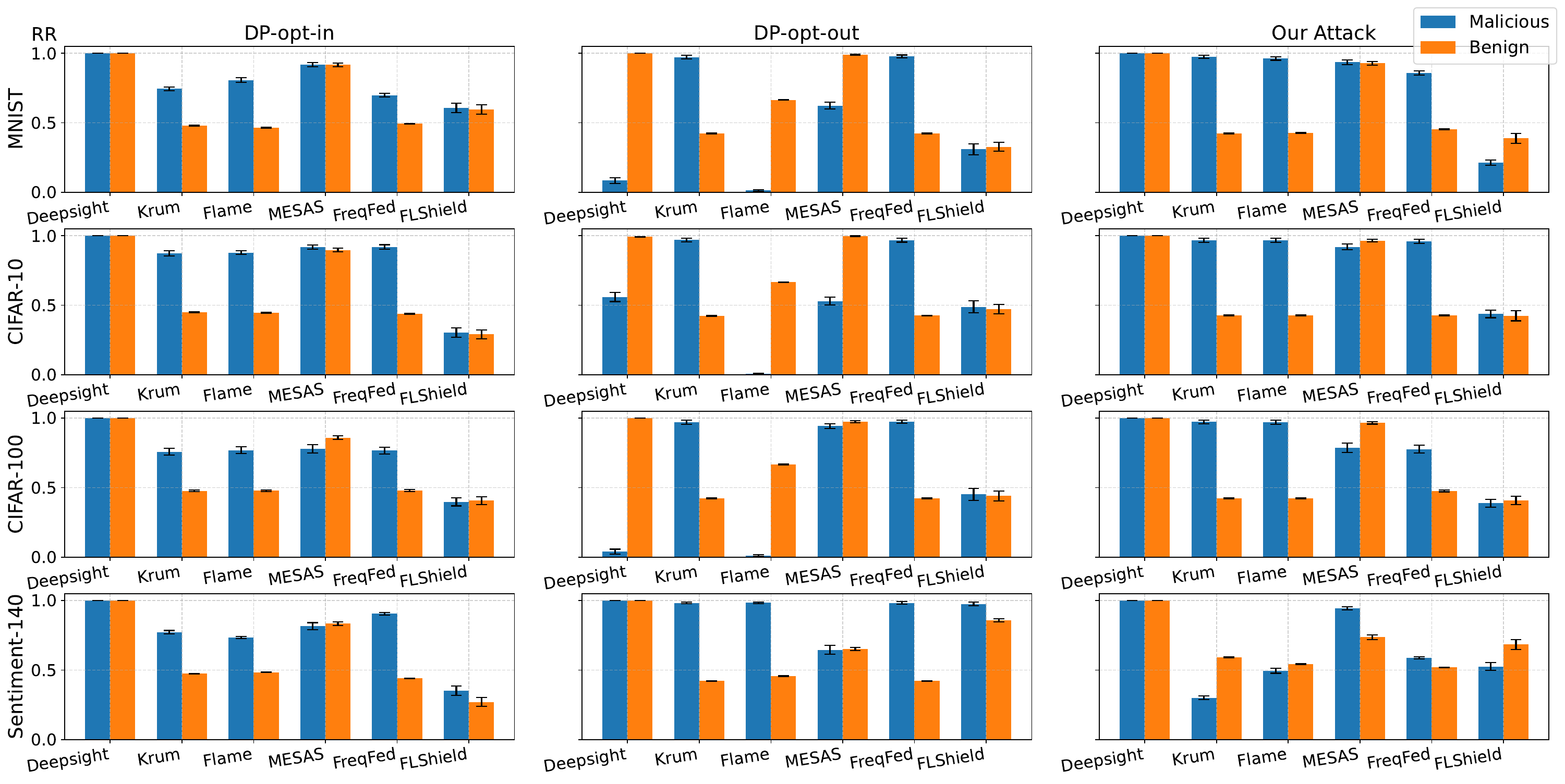}
 }\caption{Retention rate of DP-opt-in, DP-opt-out and \ring attacks in the qty non-iid setting.}
 \label{apx_fig:RR_e=5_m=6_qty}
 \vspace{-10pt}
\end{figure*}

\begin{table*}[htbp]
\centering
\scriptsize
\setlength{\tabcolsep}{4.15pt}
\caption{ASR, Acc (\%) and RR (malicious/benign,\%) of \ring, DP-opt-in, and DP-opt-out attacks under visible-trigger backdoor attack (VBA), Distributed backdoor attack (DBA) and Neurotoxin Backdoor Attack (NBA) on MNIST $(\epsilon=1)$.}
\label{tab:ASR_vary_backdoor}
\resizebox{\textwidth}{!}{%
\begin{tabular}{c|cccccccccccccc}
\toprule
\multirow{2}{*}{Backdoor}  & \multirow{2}{*}{} & {No Defense}& \multicolumn{2}{c}{Deepsight} & \multicolumn{2}{c}{Krum} & \multicolumn{2}{c}{Flame} & \multicolumn{2}{c}{MESAS} & \multicolumn{2}{c}{FreqFed} & \multicolumn{2}{c}{FLShield} \\
& & ASR / Acc& ASR / Acc&RR &ASR / Acc&RR &ASR / Acc&RR &ASR / Acc&RR&ASR / Acc&RR &ASR / Acc&RR\\
\midrule
\multirow{3}{*}{VBA} & Opt-in & 77.27 / 79.32 & 83.51 / 78.92 & 100 / 100  & 41.85 / 64.75 & 41 / 56 & 65.16 / 55.39 & 41 / 56 & 71.23 / 68.53 & 90 / 79 & 60.68 / 64.30 & 42 / 56  & 90.61 / 11.38 & 38 / 37\\
& Opt-out & 99.66 / 83.40 & 13.29 / 82.41 & 4 / 100 & 99.98 / 77.35 & 98 / 42 & 12.81 / 64.46 & 2 / 66 & 95.95 / 71.76 & 73 / 97 & 99.98 / 79.22 & 97 / 42  & 48.14 / 10.68 & 37 / 37 \\
& Ours & \cellcolor{gray!15}99.45 / \textbf{84.13} & \cellcolor{gray!15}\textbf{99.29} / \textbf{86.13} & 100 / 100 & 99.91 / 75.62 & 98 / 42 & \cellcolor{gray!15}\textbf{99.91} / \textbf{72.44} & 94 / 43 & \cellcolor{gray!15}\textbf{99.45} / \textbf{79.42} & 94 / 93 & 99.85 / 77.31 & 89 / 45 & \cellcolor{gray!15}\textbf{98.92} / \textbf{34.64} & 40 / 38 \\\cline{1-2}\cline{3-15} 
\multirow{3}{*}{DBA} & Opt-in & 21.72 / 80.55 & 18.83 / 77.06 & 100 / 100 & 22.06 / 58.23 & 42 / 56 & 25.96 / 57.26 & 41 / 57 & 31.94 / 58.45 & 88 / 79 & 23.72 / 64.63 & 41 / 57 & 77.29 / 11.44 & 39 / 40 \\
& Opt-out & 24.75 / 77.91 & 11.75 / 78.95 & 4 / 100 & 70.65 / 63.21 & 98 / 42 & 12.36 / 55.65 & 2 / 66 & 25.17 / 80.72 & 90 / 99 & 71.78 / 61.11 & 98 / 42 & 75.74 / 11.36 & 41 / 41 \\
& Ours & \cellcolor{gray!15}\textbf{33.12} / 80.43 & \cellcolor{gray!15}\textbf{20.77} / 81.32 & 100 / 100 & \cellcolor{gray!15}\textbf{88.66} / 70.58 & 98 / 42 & \cellcolor{gray!15}\textbf{78.17} / 51.47 & 94 / 43 & \cellcolor{gray!15}\textbf{38.00} / 81.75 & 92 / 92 & \cellcolor{gray!15}\textbf{72.48} / 69.82 & 88 / 45 & \cellcolor{gray!15}\textbf{93.22} / 18.01 & 43 / 36 \\\cline{1-2}\cline{3-15}
\multirow{3}{*}{NBA} & Opt-in & 28.03 / 51.15 & 4.35 / 77.70 & 100 / 100 & 1.88 / 67.41 & 42 / 56 & 4.09 / 56.42 & 41 / 57 & 39.62 / 40.15 & 88 / 79 & 1.86 / 68.92 & 41 / 57 & 88.81 / 11.88 & 39 / 40 \\
& Opt-out & 98.67 / 81.34 & 0.99 / 79.83 & 4 / 100 & 99.98 / 66.12 & 98 / 42 & 5.92 / 59.62 & 2 / 66 & 95.24 / 69.93 & 90 / 99 & 99.94 / 68.87 & 98 / 42 & 59.12 / 10.65 & 41 / 41 \\
& Ours & 98.42 / 81.13 & \cellcolor{gray!15}\textbf{98.80} / \textbf{80.35} & 100 / 100 & \cellcolor{gray!15}\textbf{99.98} / 65.93 & 98 / 42 & \cellcolor{gray!15}\textbf{99.64} / 59.52 & 94 / 43 & 90.59 / \cellcolor{gray!15}\textbf{73.27} & 92 / 92 & \cellcolor{gray!15}\textbf{99.89} / 67.95 & 88 / 45 & \cellcolor{gray!15}\textbf{91.94} / \textbf{26.47} & 43 / 36 \\
\bottomrule
\end{tabular}
}
\end{table*}
\vspace{-5pt}
\section{Proof of Theorem \ref{thm:effectiveness}}
\label{app:Proof_2}
\vspace{-5pt}
\begin{proof}
     Let $\mathcal{G}_l$ denote the $l$-th group. For each client $i \in \mathcal{G}_l$, the associated noise signal can be denoted as $ \zeta_i = (z_i - \frac{1}{m_l}\sum_{k\in{\mathcal{G}_l}}z_k)$
    where $z_k\sim \mathcal{N}(0, \sigma^2I_d).$ Suppose $k$ out of $m_l$ clients in $\mathcal{G}_l$ are kept. The probability for this event is $P_k = \binom{m_l}{k} f^k (1-f)^{m_l-k}$, for $1 \leq k < m_l$. The group’s aggregated error is
    \begin{align}\nonumber
        &\textrm{Err}_{\mathcal{G}_l,\mathcal{S}_t}=\sum_{i\in\mathcal{S}_t}\zeta_i\\ \nonumber
        &=\sum_{i\in\mathcal{S}_t}(z_i-\frac{1}{m_l}\sum_{j=1}^{m_l}z_j)\\\nonumber
        &=(\sum_{i\in\mathcal{S}_t}z_i-\frac{k}{m_l}\sum_{j=1}^{m_l}z_j) \\ \nonumber
        &=\sum_{i\in\mathcal{S}_t}z_i-\frac{k}{m_l}(\sum_{i\in\mathcal{S}_t}z_i+\sum_{j\notin\mathcal{S}_t}z_j) \\ \nonumber
        &=(1-\frac{k}{m_l})\sum_{i\in\mathcal{S}_t}z_i - \frac{k}{m_l}\sum_{j\notin\mathcal{S}_t}z_j.
    \end{align}
Since all $z_j$ are independent with zero mean, and each group is also independent, we have $\mathbb{E}||z_i||^2=d\sigma^2$, where $d$ is the dimensionality of the noise vector. For a given group, if $k$ clients remain after detection, the squared norm of the aggregated error from that group is
    \begin{align}\nonumber
        \mathbb{E}||\textrm{Err}_{\mathcal{G}_l,\mathcal{S}_t}||^2 &= d\sigma^2(k(1-\frac{k}{m_l})^2 \\ \nonumber
        &+(m_l - k)(\frac{k}{m_l})^2) \\ \nonumber
        &=d \sigma^2\cdot \frac{k(m_l-k)}{m_l}. \nonumber
    \end{align}
    Thus, we have the normalized sum of expected total error over all groups and all $k$ by the number of all the remaining malicious updates as follows
    \begin{align}\nonumber
        \mathbb{E}_{\ring}\bigl[\|\mathrm{Err}(f)\|^2\bigr] &= \frac{g}{S_t^2} \sum_{k=1}^{m_l-1} P_k \cdot 
        \mathbb{E}\left\|\textrm{Err}_{\mathcal{G}_l,\mathcal{S}_t}\right\|^2 \\ \nonumber
        &\approx \frac{g}{(mf)^2} \sum_{k=1}^{m_l-1} P_k \cdot \mathbb{E}\left\|\textrm{Err}_{\mathcal{G}_l,\mathcal{S}_t}\right\|^2 \\ \nonumber
        &= \frac{1}{(mf)^2}\cdot g \sum_{k=1}^{m_l-1} \binom{m_l}{k}f^k(1-f)^{m_l-k} \\ \nonumber
        &\cdot d\sigma^2\frac{k(m_l-k)}{m_l} \\ \nonumber
        &=\frac{m}{m_l(mf)^2}\cdot\mathbb{E}[x(m_l-x)]\cdot d\sigma^2 / m_l, \\ \nonumber
    \end{align}
    where $x\sim \textrm{Binomial}(m_l,f)$. Notice that $\mathbb{E}[x(m_l-x)] = m_l\mathbb{E}[x]-\mathbb{E}[x^2]$ and $k \sim \textrm{Binomial}(m_l, f)$. We have $\mathbb{E}[x]=m_lf$ and $\mathbb{E}[x^2] = Var(x)+(\mathbb{E}[x])^2=m_lf(1-f)+(m_lf)^2$. Therefore, $\mathbb{E}[x(m_l-x)] = m_lf(1-f)(m_l-1)$.
    Thus, we have
    \begin{align}\nonumber
        &\mathbb{E}_{\ring}\bigl[\|\mathrm{Err}(f)\|^2\bigr] \approx\frac{m}{m_l(mf)^2}\cdot\mathbb{E}[x(m_l-x)]\cdot d\sigma^2 / m_l \\ \nonumber
        &=\frac{m}{m_l(mf)^2} \cdot m_lf(1-f)(m_l-1) \cdot d\sigma^2 / m_l \\ \nonumber
        &=\frac{d\sigma^2}{m}\cdot\frac{m_l-1}{m_l}\cdot \frac{1-f}{f}.
    \end{align}
\end{proof}

\begin{figure*}[htbp]
\centering
\vspace{-5pt}
\scalebox{0.9}[0.65]{
\includegraphics[scale=0.35]{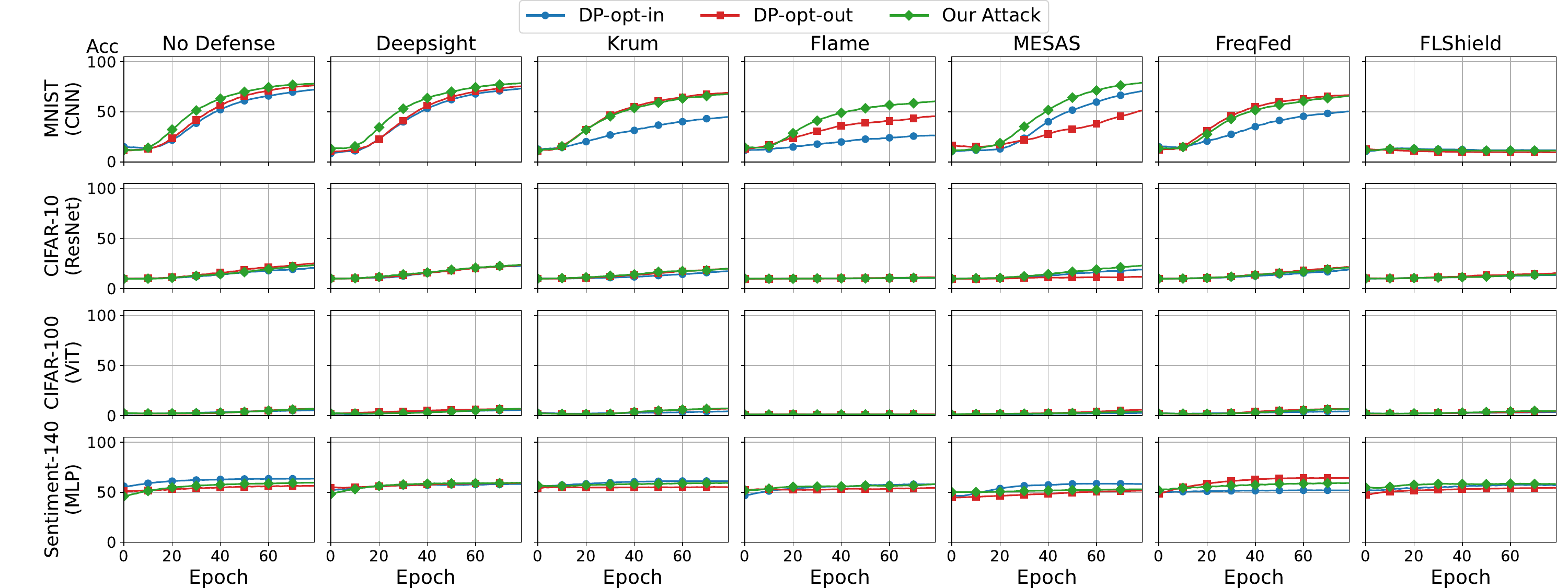}
 }\caption{Acc of DP-opt-in, DP-opt-out and \ring attacks in the qty non-iid setting.}
 \label{apx_fig:Acc_e=5_m=6_qty}
 \vspace{-5pt}
\end{figure*}
\vspace{-15pt}
\section{Proof of Theorem \ref{thm:err_compare}}
\label{app:Proof_3}

\begin{proof}
    For the DP-opt-in attack, each retained malicious client contributes an i.i.d. noise vector $z_i \sim \mathcal{N}(0,\sigma^2 I_d)$. Let $\mathcal{S}_t$ be the set of retained malicious clients. The aggregated error over $\mathcal{S}_t$ is
    \begin{equation}
      \mathrm{Err}_{\textit{DP-opt-in}}(f)=
      \begin{cases}
          \frac{1}{|\mathcal{S}_t|}\sum_{i\in \mathcal{S}_t} z_i, & |\mathcal{S}_t|>0,\\[4pt]
          0, & |\mathcal{S}_t|=0.\nonumber
      \end{cases}
    \end{equation}
    Therefore, when $|\mathcal{S}_t| > 0$, we have
    \begin{align}
        \frac{1}{|\mathcal{S}_t|}\sum_{i\in \mathcal{S}_t} z_i \sim \mathcal{N}\left(0,\frac{\sigma^2}{|\mathcal{S}_t|} I_d\right).\nonumber
    \end{align}
    Hence
    \begin{align}
        \mathbb{E}_{\textit{DP-opt-in}}\bigl[\|\mathrm{Err}(f)\|^2\bigr]
        = \frac{d\sigma^2}{|\mathcal{S}_t|} \approx \frac{d\sigma^2}{mf} \nonumber 
    \end{align}
    According to Theorem~\ref{thm:effectiveness}, we have
    \begin{align}
        \mathbb{E}_{\ring}\bigl[\|\mathrm{Err}(f)\|^2\bigr]
        \approx \frac{d\sigma^2}{m}\cdot \frac{m_l-1}{m_l}\cdot \frac{1-f}{f}.\nonumber
    \end{align}
    As a result, 
    \begin{align}
        \frac{\mathbb{E}_{\ring}\bigl[\|\mathrm{Err}(f)\|^2\bigr]}{\mathbb{E}_{\textit{DP-opt-in}}\bigl[\|\mathrm{Err}(f)\|^2\bigr]}
        &\approx
        \frac{\frac{d\sigma^2}{m}\cdot \frac{m_l-1}{m_l}\cdot \frac{1-f}{f}}
             {\frac{d\sigma^2}{mf}}
        \nonumber\\
        &=
        \frac{m_l-1}{m_l}\cdot (1-f).\nonumber
    \end{align}
\end{proof}

\begin{table}[bthp]
\vspace{-5pt}
\setlength\tabcolsep{1pt}
\centering
\caption{Datasets and Models.}
\label{tab:models}
\begin{tabular}{c|ccccccccccccccc}
\hline
\multicolumn{1}{l|}{Dataset} && \multicolumn{2}{c}{\#training} &&& \multicolumn{2}{c}{\#testing} &&& \multicolumn{2}{c}{\#input} &&& \multicolumn{2}{c}{\multirow{2}{*}{Model $\Theta$}} \\ 
 \multicolumn{1}{l|}{(\#categories)} && \multicolumn{2}{c}{samples} &&& \multicolumn{2}{c}{samples} &&& \multicolumn{2}{c}{size} &&& \\ \hline
\multicolumn{1}{l|}{MNIST (10)} && \multicolumn{2}{c}{60,000} &&& \multicolumn{2}{c}{10,000} &&& \multicolumn{2}{c}{(28,28,1)} &&&\multicolumn{2}{c}{3-Block CNN}\\
\multicolumn{1}{l|}{CIFAR-10 (10)}  && \multicolumn{2}{c}{50,000} &&& \multicolumn{2}{c}{10,000} &&& \multicolumn{2}{c}{(32,32,3)} &&&\multicolumn{2}{c}{ResNet-18}\\
\multicolumn{1}{l|}{CIFAR-100 (100)}  && \multicolumn{2}{c}{50,000} &&& \multicolumn{2}{c}{10,000} &&& \multicolumn{2}{c}{(32,32,3)} &&&\multicolumn{2}{c}{ViT-Tiny}\\
\multicolumn{1}{l|}{Sentiment-140 (2)} && \multicolumn{2}{c}{64,000} &&& \multicolumn{2}{c}{16,000} &&& \multicolumn{2}{c}{(100,)} &&&\multicolumn{2}{c}{Emb-128 + FC-1}\\ \hline
\end{tabular}
\vspace{-10pt}
\end{table}

\section{Datasets and Models}\label{app_sec:datasets_models}
Table~\ref{tab:models} summarizes the datasets, training configurations, and model architectures. For MNIST, we adopt a 3-block CNN; for CIFAR-10, ResNet-18; and for CIFAR-100, a ViT-Tiny~\cite{touvron2021training} adapted for CIFAR. The ViT-Tiny uses $4\times4$ patches, an embedding dimension of $192$, $8$ Transformer blocks, $3$ attention heads, and an MLP dimension of $768$. For Sentiment-140, we use a $128$-dimensional embedding layer followed by a fully connected output layer for binary sentiment classification.

\section{Additional Results}\label{app_sec:results_backdoor}
Figure~\ref{fig:Acc_e=5_m=6} presents the main task accuracy under all attacks. The accuracy across different backdoor variants remains similar, indicating that our attack does not degrade model performance and behaves consistently with baseline backdoor attacks.

Figures~\ref{apx_fig:ASR_e=5_m=6_qty}–\ref{apx_fig:RR_e=5_m=6_dir} report the ASR, Acc, and retention rates for different backdoor attacks under multiple defenses across all datasets in the qty non-iid and dir non-iid settings, respectively. The results are consistent with the conclusions in Section~\ref{sec:results_ring}.

\begin{figure*}[htbp]
\centering
\vspace{-5pt}
\scalebox{0.9}[0.65]{
\includegraphics[scale=0.35]{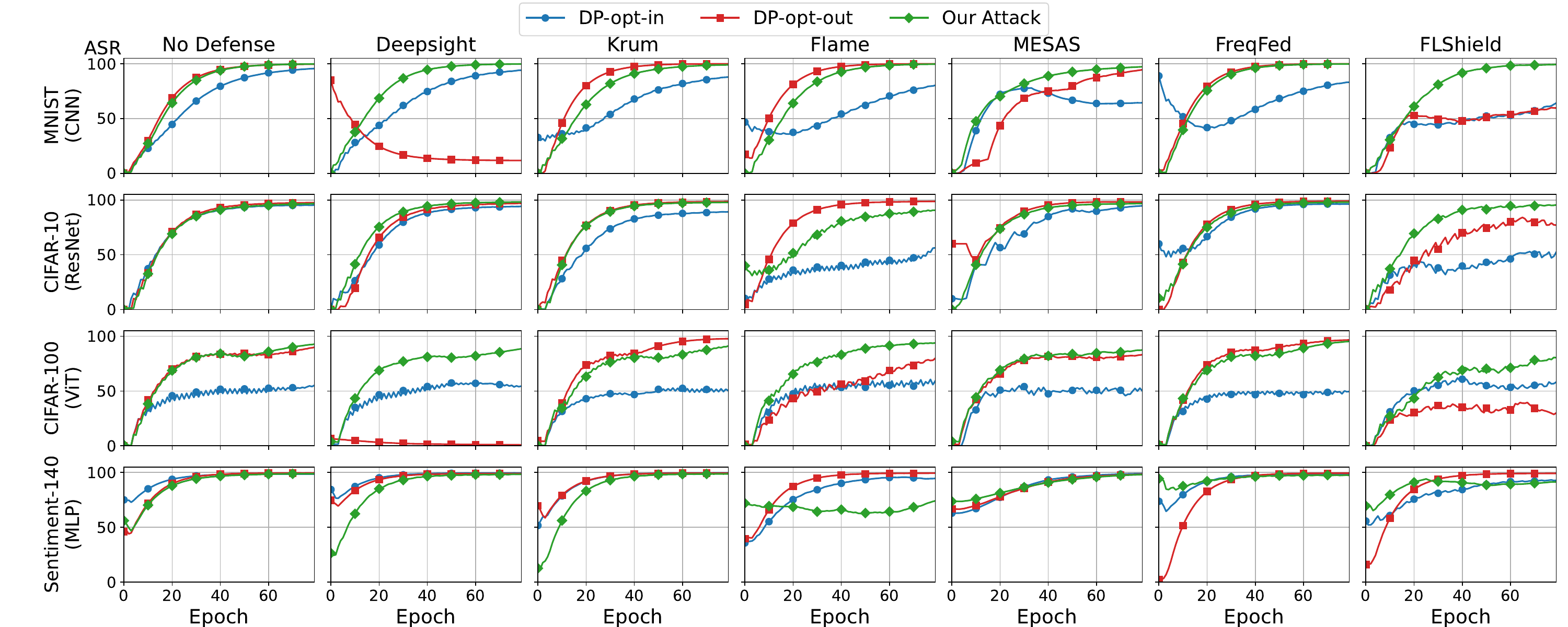}
 }\caption{ASR of DP-opt-in, DP-opt-out, and \ring attacks in the dir non-iid setting.}
 \label{apx_fig:ASR_e=5_m=6_dir}
\vspace{-20pt}
\end{figure*}

\begin{figure*}[htbp]
\centering
\vspace{-5pt}
\scalebox{0.9}[0.65]{
\includegraphics[scale=0.35]{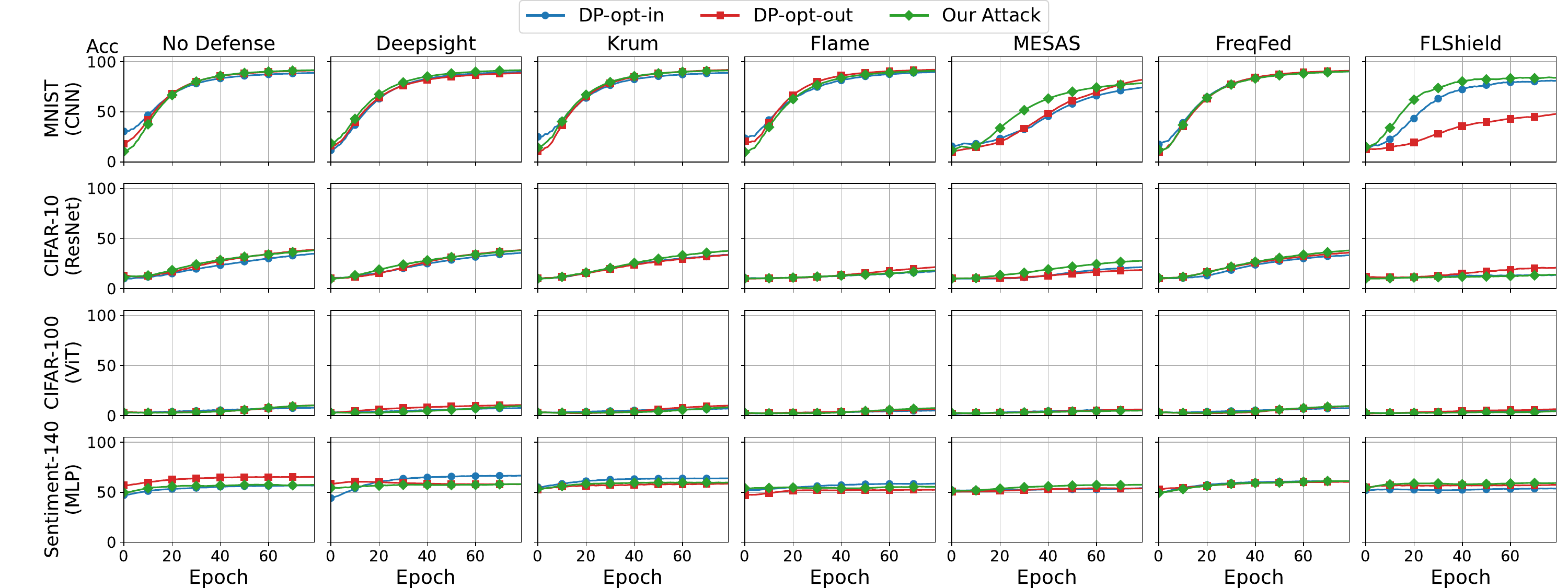}
 }\caption{Acc of DP-opt-in, DP-opt-out and \ring attacks in the dir non-iid setting.}
 \label{apx_fig:Acc_e=5_m=6_dir}
 \vspace{-20pt}
\end{figure*}

\begin{figure*}[htbp]
\centering
\vspace{-5pt}
\scalebox{0.9}[0.65]{
\includegraphics[scale=0.304]{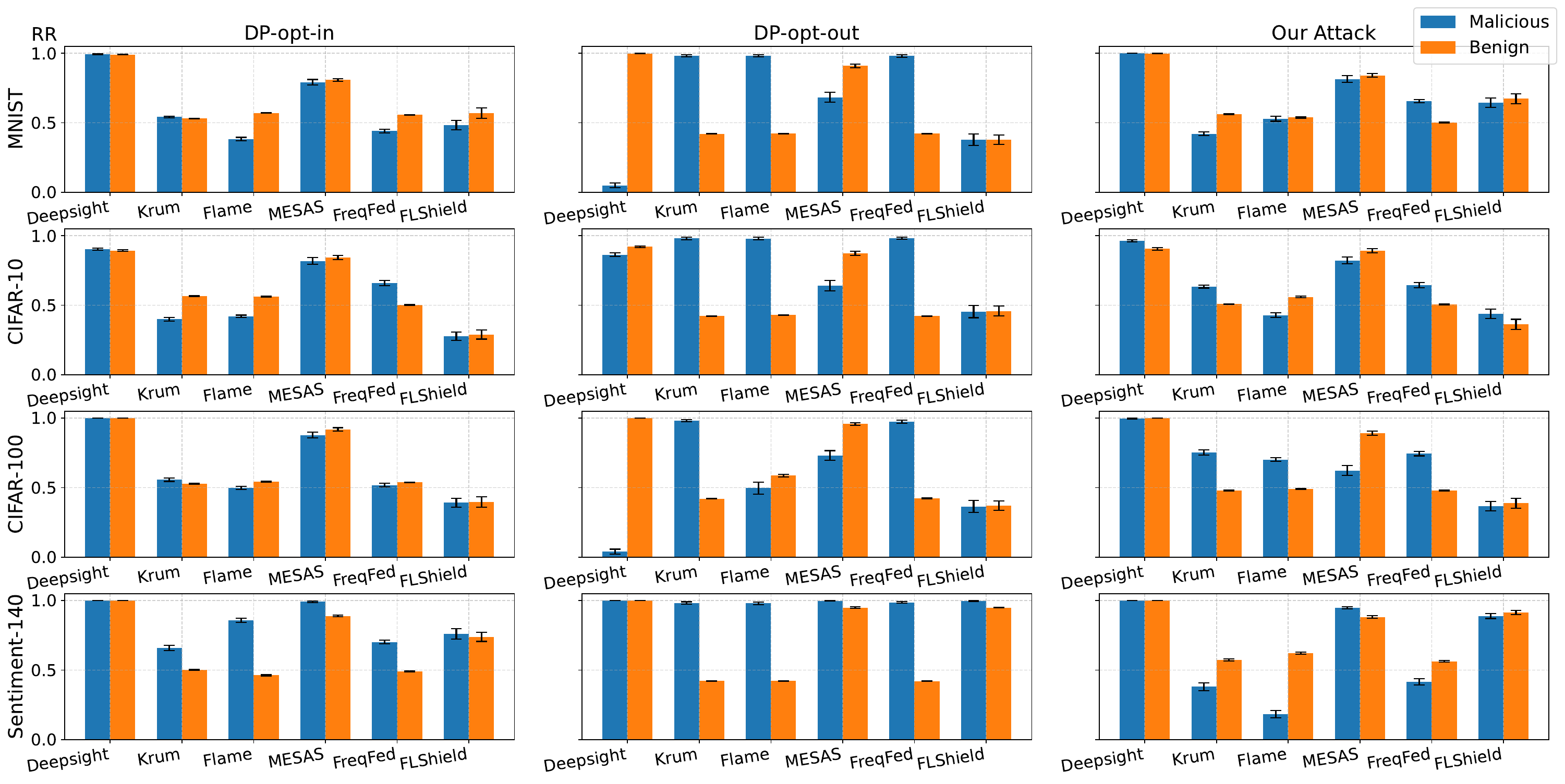}
 }\caption{Retention rate of DP-opt-in, DP-opt-out and \ring attacks in the dir non-iid setting.}
 \label{apx_fig:RR_e=5_m=6_dir}
 \vspace{-20pt}
\end{figure*}

\subsection{Random Client Participation}
\label{app:random_participation}
Following Section~\ref{sec:experimental_setting}, we use the prob non-iid setting, but randomly sample 30 of 120 clients per round without replacement with global malicious-client fractions $\beta\in\{5\%,10\%,20\%\}$. After four clean warm-up rounds, a sole malicious participant uses DP-opt-in with poisoning and normal DP noise; two or more form one \ring group. MNIST uses $\epsilon=1$, while CIFAR-10 retains $\epsilon=20$.

\begin{table}[htbp]
\centering
\setlength{\tabcolsep}{4.15pt}
\caption{Cumulative distribution of malicious-client participation. $m_t\ge2$ denotes attack activation.}
\label{tab:random_participation_distribution}
\begin{tabular}{c|ccc}
\toprule
$\beta$ & $m_t\ge2$ & $m_t\ge4$ & $m_t\ge6$ \\
\midrule
5\% & 49.74\% & 4.21\% & 0.00\% \\
10\% & 83.95\% & 32.89\% & 6.58\% \\
20\% & 99.47\% & 90.00\% & 58.95\% \\
\bottomrule
\end{tabular}
\end{table}

\begin{table}[htbp]
\centering
\setlength{\tabcolsep}{4.15pt}
\caption{ASR, Acc (\%) of \ring under random client participation.}
\label{tab:random_participation}
\begin{tabular}{c|ccc}
\toprule
\multirow{2}{*}{$\beta$} & \multirow{2}{*}{Defense} & MNIST & CIFAR-10 \\
 & & ASR / Acc & ASR / Acc \\
\midrule
\multirow{4}{*}{5\%} & No Defense & 37.51 / 86.60 & 74.07 / 35.20 \\
 & DeepSight & 27.66 / 86.88 & 81.97 / 35.30 \\
 & Flame & 71.64 / 72.58 & 84.56 / 20.34 \\
 & MESAS & 21.00 / 83.21 & 53.83 / 25.76 \\
\cline{1-4}
\multirow{4}{*}{10\%} & No Defense & 90.99 / 86.32 & 92.16 / 35.44 \\
 & DeepSight & 93.69 / 86.45 & 94.42 / 35.37 \\
 & Flame & 98.30 / 74.81 & 92.73 / 19.44 \\
 & MESAS & 80.96 / 83.74 & 92.05 / 22.77 \\
\cline{1-4}
\multirow{4}{*}{20\%} & No Defense & 99.69 / 86.92 & 97.53 / 36.10 \\
 & DeepSight & 99.53 / 86.88 & 98.18 / 36.34 \\
 & Flame & 99.78 / 76.32 & 97.29 / 19.52 \\
 & MESAS & 99.29 / 83.47 & 97.53 / 29.04 \\
\bottomrule
\end{tabular}
\end{table}

Table~\ref{tab:random_participation} reports mean ASR and Acc over five runs. ASR increases with $\beta$, exceeding 97\% at $\beta=20\%$ in all evaluated settings, whereas performance at $\beta=5\%$ varies substantially across defenses. In the tested settings with $\mathbb{E}[m_t]\ge2$ ($\beta=10\%$ and $20\%$), ASR remains high, consistent with the sequential-participation results and \ring's coordinated mechanism, which requires malicious co-participation but not a fixed participation schedule.

\subsection{Unequal Aggregation Weights}
\label{app:unequal_weights}
We follow Section~\ref{sec:experimental_setting} and evaluate \ring without additional defenses on MNIST ($\epsilon=5$) and CIFAR-10 ($\epsilon=20$) under the prob non-iid setting, using the original participation schedule with six malicious clients per attack round. We perturb aggregation weights as $w_i=n_i u_i/\sum_j n_j u_j$, where $n_i$ is the local data size and $j$ ranges over participating clients. The multipliers are applied to both benign and malicious clients. For $\lambda\in\{0.10,0.25,0.50\}$, independent per-round multipliers follow $u_i\sim\mathrm{Beta}(\kappa,\kappa)$ with $\kappa=(\lambda^{-2}-1)/2$. Larger $\lambda$ increases the dispersion of the multipliers and thus the imbalance of aggregation weights. At $\lambda=0$, we set $u_i=1$, recovering data-size-weighted FedAvg rather than equal weights.

\begin{table}[htbp]
\centering
\small
\setlength{\tabcolsep}{4.15pt}
\caption{ASR, Acc (\%) of \ring under unequal aggregation weights, averaged over five runs.}
\label{tab:unequal_weights}
\begin{tabular}{c|cc}
\toprule
\multirow{2}{*}{$\lambda$} & MNIST & CIFAR-10 \\
 & ASR / Acc & ASR / Acc \\
\midrule
0.00 & 100.00 / 90.93 & 93.85 / 27.35 \\
0.10 & 100.00 / 90.95 & 93.85 / 27.49 \\
0.25 & 100.00 / 90.94 & 93.74 / 27.51 \\
0.50 & 100.00 / 90.93 & 93.58 / 26.31 \\
\bottomrule
\end{tabular}
\end{table}

Table~\ref{tab:unequal_weights} shows that ASR remains at 100\% on MNIST and 93.58--93.85\% on CIFAR-10 across the tested values of $\lambda$. MNIST Acc remains at 90.93--90.95\%, while CIFAR-10 Acc decreases from 27.35\% to 26.31\% between $\lambda=0$ and $\lambda=0.50$.

\subsection{Smaller Privacy Budget}
\label{app:low_epsilon}

Following Section~\ref{sec:experimental_setting}, we evaluate $\epsilon=0.5$ on MNIST and CIFAR-10 under the prob non-iid setting. Table~\ref{tab:low_epsilon} reports mean ASR and Acc over five runs.

\begin{table}[htbp]
\centering
\footnotesize
\setlength{\tabcolsep}{4.15pt}
\caption{ASR, Acc (\%) at $\epsilon=0.5$.}
\label{tab:low_epsilon}
\begin{tabular}{c|c|ccc}
\toprule
\multirow{2}{*}{Dataset} & \multirow{2}{*}{Defense} & DP-opt-in & DP-opt-out & \ring \\
 & & ASR / Acc & ASR / Acc & ASR / Acc \\
\midrule
\multirow{4}{*}{MNIST} & No Defense & 99.96 / 11.35 & 98.60 / 21.33 & 99.66 / 12.99 \\
 & DeepSight & 99.92 / 11.36 & 59.53 / 10.76 & 97.51 / 19.98 \\
 & Flame & 99.37 / 11.37 & 19.91 / 10.37 & 99.95 / 11.37 \\
 & MESAS & 92.05 / 14.49 & 96.79 / 27.06 & 90.95 / 21.57 \\
\midrule
\multirow{4}{*}{CIFAR-10} & No Defense & 68.21 / 12.08 & 86.18 / 12.55 & 86.17 / 12.54 \\
 & DeepSight & 63.79 / 12.17 & 0.00 / 12.10 & 86.00 / 12.73 \\
 & Flame & 0.02 / 10.69 & 0.02 / 10.62 & 4.48 / 10.39 \\
 & MESAS & 39.99 / 10.08 & 85.56 / 12.45 & 82.36 / 12.27 \\
\bottomrule
\end{tabular}
\end{table}

Despite high ASR in most settings, \ring's clean Acc falls to 11.37--21.57\% on MNIST and 10.39--12.73\% on CIFAR-10, severely limiting model utility. The CIFAR-10 reference without attack or defense also achieves only 11.99\% Acc. Under Flame on CIFAR-10, which adds further noise, backdoor performance also deteriorates: ASR falls to 4.48\%, despite malicious and benign RR remaining 100\% and 41.67\%, respectively. These results are consistent with severe degradation of both clean-task and backdoor learning under this stringent privacy setting.

\subsection{Privacy Accounting}
\label{app:privacy_accounting}
Table~\ref{tab:privacy_accounting} reports calibration parameters using the repository's TensorFlow Privacy--derived RDP accountant, with local sampling rate $q=1$ and $\delta=10^{-5}$. $\epsilon$ is the privacy budget used for calibration, $\sigma_m$ is the noise multiplier, and $T$ is the total number of composed steps.

\begin{table}[htbp]
\centering
\small
\setlength{\tabcolsep}{4.15pt}
\caption{Privacy calibration parameters.}
\label{tab:privacy_accounting}
\begin{tabular}{c|ccc}
\toprule
Dataset & $\epsilon$ & $\sigma_m$ & $T$ \\
\midrule
\multirow{3}{*}{MNIST} & 0.5 & 85.723778 & 125 \\
 & 1 & 45.228783 & 125 \\
 & 5 & 10.665328 & 125 \\
\midrule
\multirow{2}{*}{CIFAR-10} & 0.5 & 121.231730 & 250 \\
 & 20 & 4.820832 & 250 \\
\midrule
CIFAR-100 & 20 & 4.820832 & 250 \\
Sentiment-140 & 5 & 10.665328 & 125 \\
\bottomrule
\end{tabular}
\par\smallskip
\begin{minipage}{\columnwidth}
\footnotesize
\end{minipage}
\end{table}

\clearpage
\section{Ethics Considerations}
\label{app:ethics}
This work studies a backdoor attack to identify weaknesses in DP-FL defenses. All experiments run offline on public benchmark datasets in simulated FL environments, without interacting with live third-party systems. We do not recruit participants, collect new human-subject data, or attempt to identify individuals; evaluation is limited to aggregate model performance. The attack concerns a general design weakness rather than a specific deployed product, so we disclose it through this paper and discuss potential mitigations. Any released artifacts are intended for reproducibility and defensive research.

\end{appendices}

\end{document}